\documentclass{article}

\usepackage{amsmath,amssymb,amsfonts}

\usepackage{hyperref}
\hypersetup{
    pdfnewwindow=true,     
    colorlinks=true,     
    linkcolor=blue,     
    citecolor=blue,     
    filecolor=magenta,     
    urlcolor=cyan     
}

\usepackage[capitalize]{cleveref}
\usepackage{enumitem}
\usepackage{cancel}
\usepackage{mathtools}
\usepackage{bbm,dsfont}
\usepackage{nicefrac}

\usepackage{natbib}
\usepackage[nolist,nohyperlinks]{acronym}

\newtheorem{theorem}{Theorem}
\newtheorem{proposition}[theorem]{Proposition}
\newtheorem{lemma}[theorem]{Lemma}
\newtheorem{corollary}[theorem]{Corollary}

\newcommand{\KL}{\operatorname{KL}}
\newcommand{\kl}{\operatorname{kl}}

\newcommand{\SK}{\mathcal{K}}
\newcommand{\DP}{\mathrm{DP}}

\newcommand{\PP}{\operatorname{\mathbb{P}}} 
\newcommand{\EE}{\operatorname{\mathbb{E}}} 
\newcommand{\Law}{\operatorname{Law}} 
\newcommand{\Var}{\operatorname{Var}} 

\newcommand{\N}{\mathbb{N}}  
\newcommand{\R}{\mathbb{R}}  

\DeclareMathOperator*{\argmin}{arg\,min}

\newcommand{\ind}[1]{\mathbf{1}_{#1}} 

\newcommand{\abs}[1]{\lvert #1 \rvert} 
\newcommand{\bigabs}[1]{\bigl\lvert #1 \bigr\rvert} 

\newcommand{\tr}{\operatorname{tr}} 

\newcommand*{\quotientspace}[2]
{\ensuremath{\raisebox{.25ex}{\ensuremath{#1}} \hspace{-.35ex} \raisebox{-.25ex}{/} 
\hspace{-.05ex} \raisebox{-.85ex}{\ensuremath{#2}}} \hspace{.2ex} }
\newcommand*{\textquotientspace}[2]
{\ensuremath{\raisebox{.05ex}{\ensuremath{#1}} \hspace{-.30ex} \raisebox{-.15ex}{/} 
\hspace{.10ex} \raisebox{-.55ex}{\ensuremath{#2}}}}

\newcommand{\upsum}
{\ensuremath \mkern5.4mu \overline{\vphantom{S}\mkern8mu} \mkern-11mu S}
\newcommand{\lowsum}
{\ensuremath \mkern3.4mu \underline{\vphantom{S}\mkern8mu} \mkern-9mu S}
\newcommand{\upint}
{\ensuremath \mkern10.4mu \overline{\vphantom{\int}\mkern7mu} \mkern-21mu\int}
\newcommand{\upintwdomain}[1]
{\ensuremath \mkern10.4mu \overline{\vphantom{\int}\mkern7mu} \mkern-21mu\int_{#1}}
\newcommand{\textupint}
{\textstyle \mkern5mu \overline{\vphantom{\int}\mkern6mu} \mkern-14mu\intop\mkern0mu}
\newcommand{\textupintwdomain}[1]
{\textstyle \mkern5mu \overline{\vphantom{\int}\mkern6mu} \mkern-14mu\intop_{#1}}
\newcommand{\lowint}
{\ensuremath \underline{\vphantom{\int}\mkern7mu} \mkern-9mu\int}
\newcommand{\lowintwdomain}[1]
{\ensuremath \underline{\vphantom{\int}\mkern7mu} \mkern-9mu\int_{#1}}
\newcommand{\textlowint}
{\textstyle \mkern3mu \underline{\vphantom{\int}\mkern6mu} \mkern-8.5mu\intop}
\newcommand{\textlowintwdomain}[1]
{\textstyle \mkern3mu \underline{\vphantom{\int}\mkern6mu} \mkern-8.5mu\intop_{#1}}

\def\ddefloop#1{\ifx\ddefloop#1\else\ddef{#1}\expandafter\ddefloop\fi}

\def\ddef#1{\expandafter\def\csname b#1\endcsname{\ensuremath{\mathbf{#1}}}}
\ddefloop ABCDEFGHIJKLMNOPQRSTUVWXYZ\ddefloop

\def\ddef#1{\expandafter\def\csname bb#1\endcsname{\ensuremath{\mathbb{#1}}}}
\ddefloop ABCDEFGHIJKLMNOPQRSTUVWXYZ\ddefloop

\def\ddef#1{\expandafter\def\csname c#1\endcsname{\ensuremath{\mathcal{#1}}}}
\ddefloop ABCDEFGHIJKLMNOPQRSTUVWXYZ\ddefloop

\def\ddef#1{\expandafter\def\csname s#1\endcsname{\ensuremath{\mathsf{#1}}}}
\ddefloop ABCDEFGHIJKLMNOPQRSTUVWXYZ\ddefloop

\def\ddef#1{\expandafter\def\csname v#1\endcsname{\ensuremath{\boldsymbol{#1}}}}
\ddefloop ABCDEFGHIJKLMNOPQRSTUVWXYZabcdefghijklmnopqrstuvwxyz\ddefloop

\def\ddef#1{\expandafter\def\csname v#1\endcsname{\ensuremath{\boldsymbol{\csname #1\endcsname}}}}
\ddefloop {alpha}{beta}{gamma}{delta}{epsilon}{varepsilon}{zeta}{eta}{theta}{vartheta}{iota}{kappa}{lambda}{mu}{nu}{xi}{pi}{varpi}{rho}{varrho}{sigma}{varsigma}{tau}{upsilon}{phi}{varphi}{chi}{psi}{omega}{Gamma}{Delta}{Theta}{Lambda}{Xi}{Pi}{Sigma}{varSigma}{Upsilon}{Phi}{Psi}{Omega}\ddefloop

\DeclareMathOperator{\E}{\mathbb{E}}
\renewcommand{\P}{\mathbb{P}}

\newcommand{\pr}[1]{\left( #1 \right)}
\newcommand{\br}[1]{\left[ #1 \right]}
\newcommand{\cbr}[1]{\left\{ #1 \right\}}

\newcommand{\hL}{\hat{L}}

\newcommand{\repi}{^{(i)}}

\newcommand*\diff{\mathop{}\!\mathrm{d}}

\newcommand{\tp}{^{\top}}
\newcommand{\Lh}{\hat{L}}
\newcommand{\sgap}{\hat{\varepsilon}_n}
\newcommand{\lambdah}{\hat{\lambda}}
\newcommand{\bSigma}{\boldsymbol{\Sigma}}
\newcommand{\bhSigma}{\boldsymbol{\hat{\Sigma}}}
\newcommand{\bhS}{\boldsymbol{\hat{S}}}

\newcommand{\bs}{\boldsymbol{s}}

\begin{acronym}
\acro{RLS}{Regularized Least Squares}
\acro{ERM}{Empirical Risk Minimization}
\acro{RKHS}{Reproducing kernel Hilbert space}
\acro{DA}{Domain Adaptation}
\acro{PSD}{Positive Semi-Definite}
\acro{SGD}{Stochastic Gradient Descent}
\acro{OGD}{Online Gradient Descent}
\acro{SGLD}{Stochastic Gradient Langevin Dynamics}
\acro{IS}{Importance Sampling}
\acro{WIS}{Weighted Importance Sampling}
\acro{MGF}{Moment-Generating Function}
\acro{ES}{Efron-Stein}
\acro{ESS}{Effective Sample Size}
\acro{KL}{Kullback-Liebler}
\acro{DP}{Differential Privacy}
\end{acronym}

\usepackage[colorinlistoftodos, shadow,color=blue!30!white
]{todonotes}
    \usepackage[final]{neurips_2020}

\newcommand{\BlackBox}{\rule{1.5ex}{1.5ex}}  
\newenvironment{proof}{\par\noindent{\bf Proof\ }}{\hfill\BlackBox\\[2mm]}

\begin{document}

\title{PAC-Bayes Analysis Beyond the Usual Bounds}


\author{%
  Omar Rivasplata \\ 
  University College London \& DeepMind \\
  \texttt{o.rivasplata@cs.ucl.ac.uk} 
  \And
  Ilja Kuzborskij \\
  DeepMind \\
  \texttt{iljak@google.com} 
  \And
  Csaba Szepesv\'ari \\
  DeepMind \\
  \texttt{szepi@google.com}
  \And
  John Shawe-Taylor \\
  University College London \\
  \texttt{jst@cs.ucl.ac.uk}
}

\maketitle

\begin{abstract}
We focus on a stochastic learning model where the learner observes a finite set of training examples and the output of the learning process is a data-dependent distribution over a space of hypotheses. 
The learned data-dependent distribution is then used to make randomized predictions, and the high-level theme addressed here is guaranteeing the 
quality of predictions on examples that were not seen during training, i.e. generalization. 
In this setting the unknown quantity of interest is the expected risk of the data-dependent randomized predictor, for which upper bounds can be derived via a PAC-Bayes analysis, leading to PAC-Bayes bounds.

Specifically, we present a basic PAC-Bayes inequality for stochastic kernels, from which one may derive extensions of various known PAC-Bayes bounds as well as novel bounds. 
We clarify the role of the requirements of fixed `data-free' priors, bounded losses, and i.i.d. data.
We highlight that those requirements were used to upper-bound an exponential moment term, while the basic PAC-Bayes theorem remains valid without those restrictions. 
We present three bounds that illustrate the use of data-dependent priors, including one for the unbounded square loss.
\end{abstract}

\section{Introduction}
\label{sec:intro}

The context of this paper is the statistical learning model where 
the learner observes 
training data $S = \pr{Z_1, Z_2, \ldots, Z_n}$  
randomly drawn from a space of size-$n$ samples $\cS = \cZ^n$ 
(e.g. $\cZ = \R^d \times \cY$ for a supervised learning problem where the input space is $\R^d$ and the label set is $\cY$)
according to some unknown probability distribution\footnote{We write $\cM_1(\cX)$ to denote the family of probability measures over a set $\cX$, see \cref{measurability}.}
$P_n \in \cM_1(\cS)$. Typically $Z_1,\dots,Z_n$ are independent 
and share a common distribution $P_1\in \cM_1(\cZ)$.
Upon observing the training data $S$, the learner outputs a \emph{data-dependent} probability distribution $Q_S$ over a \emph{hypothesis space} $\cH$.
Notice that this learning scenario involves randomness in 
the data and the hypothesis.
In this stochastic learning model,
the randomized predictions are carried out by randomly drawing 
a fresh hypothesis for each prediction.
Therefore, we consider the performance of a probability distribution $Q$ over the hypothesis space:
the expected \emph{empirical loss} is $Q[\hL_s] = \int_{\cH} \hL_s(h) Q(dh)$, 
i.e. the $Q$-average of the standard empirical loss 
$\hL_s(h) = \hL(h,s)$ defined as $\hL(h,s) = \frac1n \sum_{i=1}^n \ell(h, z_i)$ for a fixed $h \in \cH$ and $s = (z_1,\ldots,z_n)$, where $\ell : \cH \times \cZ \to [0,\infty)$ is a given loss function.
Similarly, the expected \emph{population loss} is $Q[L] = \int_{\cH} L(h) Q(dh)$, i.e. the $Q$-average of the standard population loss $L(h) = \int_{\cZ} \ell(h, z) P_1(dz)$ for a fixed $h \in \cH$, 
where $P_1 \in \cM_1(\cZ)$ is the distribution that generates one random example.

An important component of our development is formalizing ``data-dependent distributions over $\cH$'' in a way that makes explicit their difference to fixed ``data-free'' distributions over $\cH$.

\paragraph{Data-dependent distributions as stochastic kernels.}
A data-dependent distribution over the space $\cH$ is formalized as a \emph{stochastic kernel}\footnote{This is also called a \emph{transition kernel} or \emph{probability kernel}, a well-known concept in the literature on stochastic processes, see e.g. \cite{kallenberg2017,MeynTweedie2009,EthierKurtz1986}.}
from $\cS$ to $\cH$, which is defined as a mapping\footnote{The space of size-$n$ samples $\cS$ is equipped with a sigma algebra that we denote $\Sigma_{\cS}$, and the hypothesis space $\cH$ is equipped with a sigma algebra that we denote $\Sigma_{\cH}$. For precise definitions  see \cref{measurability}.}
$Q : \cS\times\Sigma_{\cH} \to [0,1]$ such that
(i) for each $B \in \Sigma_{\cH}$ the function $s \mapsto Q(s,B)$ is measurable; and 
(ii) for each $s \in \cS$ the function $Q_s : B \mapsto Q(s,B)$ is a probability measure over $\cH$. 
We write $\cK(\cS,\cH)$ to denote the set of all stochastic kernels 
from $\cS$ to ---distributions over--- $\cH$.
We reserve the notation $\cM_1(\cH)$ for the set of `data-free' distributions over $\cH$. 
Notice that $\cM_1(\cH) \subset \cK(\cS,\cH)$, since every `data-free' distribution can be regarded as a constant kernel.

With the notation just introduced,
$Q_S$ stands for the distribution over $\cH$ corresponding to a randomly drawn data set $S$. The stochastic kernel $Q$ can be thought of as describing a randomizing learner.
One well-known example 
is the \emph{Gibbs} learner, where $Q_S$ is of the form $Q_S(dh) \propto e^{-\gamma \hL(h,S)} \mu(dh)$ for some $\gamma > 0$, with $\mu$ a base measure over $\cH$.
Note that, besides randomized predictors, other prediction schemes may be devised from a learned distribution over hypotheses, as for instance ensemble predictors and majority vote predictors (see the related literature in \cref{s:literature} below).

A common question arising in learning theory
aims to explain the generalization ability of a learner: 
how can a learner ensure a `well-behaved' population loss?
One way to answer this question is via upper bounds on the population loss, also called \emph{generalization bounds}.
Often the focus is on the \emph{generalization gap}, which is 
the difference between the population loss and the empirical loss,
and giving upper bounds on the gap.
There are several types of generalization bounds we care about in learning theory, with variations in the way they depend on the training data $S$ and the data-generating distribution $P_n$.
The classical bounds (such as VC-bounds) depend on neither.
\emph{Distribution-dependent} bounds are expressed in terms of  quantities related to the data-generating distribution (e.g.\ population mean or variance) and possibly constants, but not the data in any way.
These bounds can be helpful to study the behaviour of a learning method on different distributions---for example, some data-generating distributions might give faster convergence rates than others.
Finally, \emph{data-dependent} bounds are expressed in terms of empirical quantities that can be computed directly from data. 
These are useful for building and comparing predictors \citep{catoni2007}, and also for ``self-bounding'' \citep{Freund1998} or ``self-certified'' \citep{perez-ortiz2020tighter} learning algorithms, which are learning algorithms that use all the available data to simultaneously provide a predictor and a risk certificate that is valid on unseen examples.

\emph{PAC-Bayesian} inequalities allow to derive distribution- or data-dependent 
generalization bounds in the context of the stochastic prediction model discussed above.
The usual PAC-Bayes analysis introduces a reference `data-free' probability measure $Q^0 \in \cM_1(\cH)$ on the hypothesis space $\cH$.
The learned data-dependent distribution $Q_S$ is commonly called a \emph{posterior}, while $Q^0$ is called a \emph{prior}.
However, in contrast to Bayesian learning, 
the PAC-Bayes prior $Q^0$ acts as an analytical device and may or may not be used by the learning algorithm,
and 
the PAC-Bayes posterior $Q_S$ is unrestricted and so it may be different from the posterior that would be obtained from $Q^0$ through Bayesian inference.
In this sense, the PAC-Bayes approach affords an extra level of flexibility in the choice of distributions, even compared to generalized Bayesian approaches \citep{bissiri2016general}.

In the following, for any given $Q \in \cK(\cS,\cH)$ and $s \in \cS$, we write $Q_s[\hL_s] = \int \hL_s(h) Q_s(dh)$ and $Q_s[L] = \int L(h) Q_s(dh)$ for the expected empirical loss and the expected population loss, respectively.
The focus of PAC-Bayes analysis is deriving bounds on 
the gap between $Q_S[L]$ and $Q_S[\hL_S]$.
For instance, the classical result of~\cite{McAllester1999} says the following:
For a fixed `data-free' distribution $Q^0 \in \cM_1(\cH)$, bounded loss function with range $[0,1]$,
stochastic kernel $Q \in \cK(\cS, \cH)$ and for any $\delta\in (0,1)$,
with probability at least $1-\delta$ over size-$n$ random samples $S$:
\begin{equation}
  \label{eq:intro:mcallester}
  Q_S[L] - Q_S[\hL_S] \leq \sqrt{\frac{1}{2n-1} \pr{\KL(Q_S \Vert Q^0) + \log\pr{ \textstyle \frac{n+2}{ \delta}}}}~.
\end{equation}
$\KL(\cdot\Vert\cdot)$ stands for the Kullback-Leibler divergence\footnote{Also known as relative entropy, see e.g. \citet{CoverThomas2006}.} 
which is defined for two given probability distributions $Q, Q^{\prime}$ over $\cH$ as follows:
$\KL(Q \Vert Q^{\prime}) = \int_{\cH} \log\pr{ dQ / dQ^{\prime} } dQ$,
where $dQ/dQ^{\prime}$ denotes the Radon-Nikodym derivative.
%
Note that PAC-Bayes bounds (e.g. McAllester's bound described above) are usually presented under a statement that says that with probability at least $1-\delta$, 
the displayed inequality holds simultaneously
for all probability distributions 
$Q$ over $\cH$, i.e. with an arbitrary $Q$ replacing $Q_S$.
Such commonly used formulation has the apparent advantage of being valid uniformly for all distributions over $\cH$, while our formulation is valid for a fixed kernel. 
At the same time, the commonly used formulation has the disadvantage of hiding the data-dependence of the `posterior' distributions used in practice, while our formulation in terms of a stochastic kernel shows explicitly the data-dependence: given the data $S$, the corresponding distribution over $\cH$ is $Q_S$.
Notice that one fixed stochastic kernel suffices in order to describe a whole parametric family of distributions (such as Gaussian or Laplace distributions, among others) with parameter values learned from data.
Since our main interest is in results for data-dependent distributions (contrasted to results for fixed `data-free' distributions), we argue in favour of the formulation based on stochastic kernels.
These have appeared in the learning theory literature under the names of Markov kernels \citep{xu2017information} or regular conditional probabilities \citep{catoni2004,catoni2007,Alquier2008}.


A large body of subsequent work focused on refining the PAC-Bayes analysis by means of alternative proof techniques and different ways to measure the gap between $Q_S[L]$ and $Q_S[\hL_S]$. 
For instance~\cite{LangfordSeeger2001} and \cite{seeger2002} gave an upper bound on the relative entropy of $Q_S[\hL_S]$ and $Q_S[L]$,
commonly called the PAC-Bayes-kl bound \citep{Seldin-etal2012}, which holds with high probability over randomly drawn size-$n$ samples $S$:
\begin{equation}
  \label{eq:intro:seeger}
  \kl(Q_S[\hL_S] \,\Vert\, Q_S[L]) \leq
  \frac1n \pr{\KL(Q_S \Vert Q^0)+\log\pr{ \textstyle \frac{n+1}{ \delta}}}\,.
\end{equation}
$\kl(\cdot\Vert\cdot)$, appearing on the left-hand side of this inequality, denotes the binary KL divergence, which is by definition the KL divergence between the Bernoulli distributions with the given parameters: 
$$
\kl(q \Vert q^\prime) = q \log(\frac{q}{q^\prime}) + (1-q)\log(\frac{1-q}{1-q^\prime}) 
\hspace{5mm}\text{for}\hspace{2mm}
q,q^\prime \in [0,1].
$$
Inequality~\eqref{eq:intro:seeger} is tighter than~\eqref{eq:intro:mcallester} due to Pinsker's inequality $2 (p - q)^2 \leq \kl(p \Vert q)$.
In fact, by a refined form of Pinsker's inequality,
namely $(p - q)^2/(2q) \leq \kl(p \Vert q)$ which is valid for $p<q$ 
(and tighter than the former when $q<0.25$),
from \cref{eq:intro:seeger} one obtains a \emph{localised} inequality%
\footnote{For $x,b,c$ nonnegative, 
$x \leq c + b\sqrt{x}$ implies $x \leq c + b\sqrt{c} + b^2$.
}
(see Eq. (6) of \cite{McAllester2003}), 
which holds with high probability%
\footnote{The notation $\lesssim$ hides universal constants and logarithmic factors.}
over randomly drawn size-$n$ samples $S$:
\begin{equation}
  \label{eq:intro:bernstein}
  Q_S[L] - Q_S[\hL_S] \lesssim
  \sqrt{ 
    \frac{Q_S[\hL_S]}{n} \, \KL(Q_S \Vert Q^0)
  }
  + \frac1n \KL(Q_S \Vert Q^0)~.
\end{equation}
PAC-Bayes bounds like \cref{eq:intro:mcallester} and \cref{eq:intro:bernstein} tell us that the population loss is controlled by a trade-off between the empirical loss and the deviation of the posterior from the prior as captured by the KL divergence.
Note that inequality~\eqref{eq:intro:bernstein} is tighter
than~\eqref{eq:intro:mcallester} when $Q_S[\hL_S] < Q_S[L] < 0.25$.
Obviously, the upper bound in \cref{eq:intro:bernstein} is dominated by the lower-order (second) term whenever the empirical loss $Q_S[\hL_S]$ is small enough, which makes this inequality very appealing for learning problems based on empirical risk minimization, where the empirical loss is driven to zero. At a high level, such kinds of data-dependent upper bounds on the generalization gap are much desirable, as their empirical terms are closely linked to---and hopefully capture more properties of---the data.
In this direction, valuable contributions were made by
\cite{tolstikhin2013pac} who obtained an empirical PAC-Bayes bound similar in spirit to~\cref{eq:intro:bernstein}, but controlled by the sample variance of the loss.
An alternative direction to get sharper empirical bounds was explored through 
\emph{tunable} bounds \citep{catoni2007,vanerven2014mini,thiemann-etal2017}, 
which involve a free parameter that offers a trade-off between the empirical error term 
and the $\KL(\text{Posterior} \Vert \text{Prior})$ term.

Despite their variety and attractive properties, 
the results discussed above (and the vast majority of the literature) share two crucial limitations: the prior $Q^0$ cannot depend on the training data $S$ and the loss function has to be bounded.
It is conceivable that in many realistic situations the population loss is effectively controlled by the $\KL$ ``complexity'' term---indeed, in most modern learning scenarios (e.g.\ training deep neural networks) the empirical loss is driven to zero.
At the same time, the choice of a fixed `data-free' prior essentially becomes a wild guess on how the posterior will look like. 
Therefore, allowing prior distributions to be data-dependent introduces 
much needed flexibility, since this opens up the possibility to minimize upper bounds in both the posterior \emph{and the prior}, which should lead to tighter empirical bounds on $Q_S[L]$ and tighter risk certificates.

These limitations have been 
removed in the PAC-Bayesian literature in special cases.
For instance, \cite{ambroladze2007tighter} and \cite{parrado2012pac} used priors that were trained on a held-out portion of the available data, thus enabling empirical bounds with PAC-Bayes priors that are data-dependent, but independent from the training set.
Priors that depend on the full training set have also been studied recently.
\citet{thiemann-etal2017} proposed to construct a prior as a
mixture of point masses at a finite number of data-dependent hypotheses trained on a $k$-fold split of the training set, effectively a data-dependent prior.
Another approach was proposed by~\cite{dziugaite2018data}: rather than 
splitting the training data,
they require the data-dependent prior $Q^0_s$ (where $Q^0 \in \cK(\cS, \cH)$) to be stable with respect to `small' changes in the composition of the $n$-tuple $s$.
As we will see shortly, there is benefit in relaxing the restrictions of the usual PAC-Bayes literature.

\section{Our Contributions}
\label{s:contributions}
In this paper we discuss a basic PAC-Bayes inequality (\cref{basic_pb_ineq} below) and
a general template for PAC-Bayesian bounds (\cref{pbb-general} below).
The formulation of both these results is based on representing data-dependent distributions as stochastic kernels.
To make a case for the usefulness of this approach, we show that
our \cref{pbb-general} encompasses many usual bounds which appear in the literature~\citep{McAllester1998,McAllester1999,seeger2002,catoni2007,thiemann-etal2017}, while at the same time it enables new PAC-Bayes inequalities. 
Importantly, our study takes a critical stand on
the ``usual assumptions'' on which PAC-Bayes inequalities are based, namely, (a) data-free prior, (b) bounded loss, and (c) i.i.d. data observations. We aim to clarify the role of these assumptions and to illustrate how to obtain PAC-Bayes inequalities in cases where these assumptions are removed. 
As we will soon see, the analysis leading to our \cref{pbb-general} shows that the PAC-Bayes priors can be data-dependent by default, and also that the underlying loss function can be unbounded by default.
Furthermore, the proof of our \cref{pbb-general} does not rely on the assumption of i.i.d. data observations, which may enable new results for statistically dependent data in future research.

For illustration,
our general PAC-Bayes theorem\footnote{Generic PAC-Bayes theorems, similar in spirit to ours, have been presented before, e.g. by \cite{audibert2004better,germain-etal2009,begin2014pac,begin2016pac}, but only with fixed `data-free' priors.} 
for stochastic kernels (\cref{pbb-general} in \cref{s:pb-theory}), in specialized form, implies that for any convex function $F : \R^2 \to \R$, for any stochastic kernels $Q,Q^0 \in \cK(\cS, \cH)$ and $\delta \in (0,1)$, with probability at least $1-\delta$ over randomly drawn $S$ one has
\begin{equation}
  \label{eq:intro:general_bound}
  F(Q_S[L], Q_S[\hL_S]) \leq
  \KL(Q_S \Vert Q^0_S)+\log(\xi(Q^0)/\delta)~,
\end{equation}
where $\xi(Q^0)$ is the exponential moment 
of $F(L(h), \hL_s(h))$, 
which is defined as follows:
\[
  \xi(Q^0) = \int_{\cS}\int_{\cH} e^{F(L(h), \hL_s(h))} Q^0_s(dh) P_n(ds) ~.
\]
Observe that \cref{eq:intro:general_bound} is defined 
for an arbitrary convex function $F$.
This way the usual bounds are encompassed:
$F(x,y) = 2n(x-y)^2$ yields a \cite{McAllester1999}-type bound, 
$F(x,y) = n \kl(y \Vert x)$ gives the bound of~\cite{seeger2002},
and $F(x,y) = n\log\pr{\frac{1}{1-x(1-e^{-\lambda})}} - \lambda n y$ gives the bound of~\cite{catoni2007}.
Furthermore, $F(x,y) = n (x-y)^2/(2x)$ leads to the so-called PAC-Bayes-$\lambda$ bound of~\cite{thiemann-etal2017}, 
or 
to the bound of \cite{rivasplata2019pac}
which holds under the usual requirements of fixed `data-free' prior $Q^0$, losses within the $[0,1]$ range, and i.i.d. data:
\begin{align}
    Q_S[L] \leq \left(
    \sqrt{ 
  Q_S[\hL_S] + \frac{\KL(Q_S \Vert Q^0) + \log(\frac{2\sqrt{n}}{\delta})}{2n} 
    }
    +
    \sqrt{ 
    \frac{\KL(Q_S \Vert Q^0) + \log(\frac{2\sqrt{n}}{\delta})}{2n} 
    } \right)^2\,.
\end{align}
As consequence of
the universality of~\cref{eq:intro:general_bound}, besides the usual bounds we may derive novel bounds, e.g. with data-dependent priors $Q^0_S$.
Conceptually, our approach splits the usual PAC-Bayesian analysis into
two components:
(i) choose $F$ to use in \cref{eq:intro:general_bound}, and 
(ii) obtain an upper bound on the exponential moment $\xi(Q^0)$.
The cost of generality is that for each specific choice of the bound (technically, a choice of a function $F$ and $Q^0$) we need to study 
the exponential moment $\xi(Q^0)$ and, in particular, provide a reasonable, possibly data-dependent upper bound on it.
We stress that the only technical step necessary for the introduction of a data-dependent prior is a bound on $\xi(Q^0)$, the rest is taken care of by \cref{eq:intro:general_bound}.
While previous works\footnote{\cite{audibert2007combining,alquier2016properties}, among others, for the case of fixed `data-free' priors.} analysed separately the exponential moment, as we do here, to the best of our knowledge they considered data-free priors only.
We think our work is the first to point out techniques to upper bound $\xi(Q^0)$ when $Q^0$ is a stochastic kernel, and to present PAC-Bayesian inequalities where the prior is  data-dependent by default. 
Our work also clarifies where / how the data-free nature of the priors was used in previous works.

We emphasize that in this paper the main focus is on using data-dependent priors in the PAC-Bayes analysis.
Again, we point out that the proof of the basic PAC-Bayes inequality (\cref{basic_pb_ineq} below) does not require fixed `data-free' priors, nor bounded loss functions nor i.i.d. data observations. 
The same can be said of \cref{pbb-general}, a consequence of \cref{basic_pb_ineq}(ii), which gives a general template for deriving PAC-Bayes bounds.
Below we discuss three generalization bounds with data-dependent priors, two of which are for bounded losses, while the third is for the unbounded square loss.

\subsection{A PAC-Bayes bound with a data-dependent Gibbs prior}

Choosing as prior an \emph{empirical Gibbs} distribution 
$Q_s^0(dh) \propto e^{-\gamma \hL(h,s)} \mu(dh)$ for some fixed $\gamma > 0$ and base measure $\mu$ over $\cH$, we derive a novel PAC-Bayes bound.
Recall that $s$ is the size-$n$ sample.
We use $F(x,y) = \sqrt{n} (x-y)$, and we prove that in this case the exponential moment $\xi(Q^0)$ satisfies
\[
  \log(\xi(Q^0)) \leq 2 \pr{1 + \frac{2\gamma}{\sqrt{n}} } + \log\pr{1 + \sqrt{e}}~.
\]
The proof (\cref{proof_xi_empirical_gibbs}) is based on the algorithmic stability argument for Gibbs densities, inspired by the proof of~\citet[Theorem 1]{kuzborskij2019distribution}.
%
Combining this with \cref{eq:intro:general_bound}, 
for any kernel
$Q \in \cK(\cS, \cH)$ 
and $\delta \in (0,1)$, with probability at least $1-\delta$ over size-$n$ i.i.d. samples $S$ we have
\begin{align}
  \label{eq:intro:gibbs_bound}
  Q_S[L] - Q_S[\hL_S]
  \leq
  \frac{1}{\sqrt{n}}
  \left(
  \KL(Q_S \Vert Q^0_S)
  + 2 \Bigl( 1+ \frac{2\gamma}{\sqrt{n}} \Bigr)
  + \log\Bigl(\frac{1 + \sqrt{e}}{\delta}\Bigr)
  \right)~.
\end{align}
Notice that this prior allowed to remove `$\log(n)$' from the usual PAC-Bayes bounds (see our Eq.~\eqref{eq:intro:mcallester} and Eq.~\eqref{eq:intro:seeger} above). 
This was one of the important contributions of \cite{catoni2007}, who also used a data-dependent Gibbs distribution, see \citet[Theorem 1.2.4, Theorem 1.3.1, \& corollaries]{catoni2007}.
%
Interestingly, the choice $Q = Q^0$ gives the smallest right-hand side in~\cref{eq:intro:gibbs_bound} (however, it does not necessarily minimize the bound on $Q_S[L]$) which leads to the following for the Gibbs learner: 
$
  Q_S[L] - Q_S[\hL_S] \lesssim 1/\sqrt{n} + \gamma/n~.
$
Notice that this latter bound has an additive $1/\sqrt{n}$ compared to the bound in expectation of~\cite{raginsky2017nonconvex}.


\subsection{PAC-Bayes bounds with d-stable data-dependent priors}

Next we discuss an approach 
to convert any PAC-Bayes bound with a usual `data-free' prior into a bound with a stable data-dependent prior, which is accomplished by generalizing a technique from \cite{dziugaite2018data}.
Essentially, they show (see \cref{sec:DP}) that for any fixed `data-free' distribution $Q^*\in \cM_1(\cH)$ and  stochastic kernel $Q^0 \in \cK(\cS, \cH)$ satisfying the $\DP(\epsilon)$ property\footnote{$\DP(\epsilon)$ stands for ``differential privacy with $\epsilon$.'' See \cref{sec:DP} for details on this property.},
one can turn the inequality
$F(Q_S[L], Q_S[\hL_S]) \leq
    \KL(Q_S \Vert Q^*)+\log(\xi(Q^*)/\delta)$
into 
\begin{equation}
  \label{eq:intro:mgf_DP}
  F(Q_S[L], Q_S[\hL_S]) \leq
    \KL(Q_S \Vert Q^0_S)+\log(2 \xi(Q^*)/\delta) 
    + \frac{n\epsilon^2}{2}
  + \epsilon\sqrt{\frac{n}{2} \log(\frac{4}{\delta})}~.
\end{equation}
In other words, if \cref{eq:intro:general_bound} holds with a data-free prior $Q^*$, then \cref{eq:intro:mgf_DP} holds with a data-dependent prior that is distributionally stable (i.e. satisfies $\DP(\epsilon)$).
Note that different choices of $F$ would lead to different bounds on $\xi(Q^*)$ ---essentially, upper bounds on the exponential moment typically considered in the PAC-Bayesian literature.
For example, taking $F(x,y) = n \kl(y\Vert x)$ one can show that
$\xi(Q^*) \leq 2\sqrt{n}$ \citep{maurer2004note}, and this leads to Theorem~4.2 of~\cite{dziugaite2018data}: if $Q^0 \in \cK(\cS, \cH)$ satisfies the $\DP(\epsilon)$ property, then for any kernel
$Q \in \cK(\cS, \cH)$ 
and $\delta \in (0,1)$, with probability at least $1-\delta$ over size-$n$ i.i.d. samples $S$ we have
\begin{align*}
  \kl(Q_S[\hL_S] \Vert Q_S[L])
  \leq
  \frac1n
  \pr{
  \KL(Q_S \Vert Q^0_S) 
  + \log(\frac{4\sqrt{n}}{\delta})
  + \frac{n\epsilon^2}{2}
  + \epsilon\sqrt{\frac{n}{2} \log(\frac{4}{\delta})}
  }~.
\end{align*}
\cref{eq:intro:mgf_DP} is a general version of this result, whose derivation is based on the notion of \emph{max-information} \citep{dwork2015generalization}.
The details of the general conversion recipe are given in \cref{sec:DP}.

\subsection{A generalization bound for the square loss with a data-dependent prior}
\label{sec:square_loss}

Our third and last contribution is a novel bound for the setting of learning linear predictors with the square loss. This will demonstrate the full power of our take on the PAC-Bayes analysis, as we will consider a regression problem with the unbounded squared loss and a data-dependent prior. In fact, our framework of data-dependent priors makes it possible to obtain the problem-dependent bound in \cref{eq:square_loss_bd} for square loss regression. We are not aware of an equivalent previous result.

In this setting, the input space is $\cX = \R^d$ and the label space $\cY = \R$.
A linear predictor is of the form $h_{w}: \R^d \to \R$ with $h_{w}(x) = w\tp x$ for $x \in \R^d$, where of course $w \in \R^d$. Hence $h_{w}$ may be identified with the weight vector $w$ and correspondingly the hypothesis space $\cH$ may be identified with the weight space $\cW =\R^d$. The size-$n$ random sample is $S = \pr{(X_1, Y_1), \ldots, (X_n, Y_n)} \in (\R^d \times \R)^n$.
The population and empirical losses are defined with respect to the square loss function:
\[
  L(w) = \frac12 \EE[(w\tp X_1 - Y_1)^2] 
  \hspace{7mm}\text{and}\hspace{7mm}
  \hL_S(w) = \frac{1}{2n} \sum_{i=1}^n (w\tp X_i - Y_i)^2~.
\]
The population covariance matrix is $\vSigma = \EE[X_1 X_1\tp] \in \R^{d \times d}$ and its eigenvalues are $\lambda_1 \geq \dots \geq \lambda_d$.
The (regularized) sample covariance matrix is $\hat{\vSigma}_{\lambda} = (X_1 X_1\tp + \dots + X_n X_n\tp)/n + \lambda \vI$ for $\lambda > 0$, with eigenvalues $\hat{\lambda}_1 \geq \dots \geq \hat{\lambda}_d$.
Note that $\hat{\lambda}_i$ are data-dependent.

Consider the prior $Q^0_{\gamma,\lambda}$ with density $q^0_{\gamma,\lambda}(w) \propto e^{-\frac{\gamma \lambda}{2} \|w\|^2}$ for some $\gamma,\lambda > 0$, that possibly depend on the data.
In this setting, we prove (\cref{app:lease_squares}) that for any posterior $Q \in \cK(\cS, \cW)$, for any $\gamma > 0$, 
and any $ \lambda > \max_i\{\lambda_i - \hat{\lambda}_i\}$,
with probability one over size-$n$ random samples $S$ we have
\begin{align}
  \label{eq:square_loss_bd}
  Q_S[L] - Q_S[\hL_S]
  \leq
  \min_{w \in \R^d} L(w)
  +
  \frac{1}{\gamma} \KL(Q_S \,||\, Q^0_{\gamma,\lambda})
  +  
  \frac{1}{2 \gamma} \sum_{i=1}^d \log\pr{\frac{\lambda}{\lambda + \hat{\lambda}_i - \lambda_i}}~.
\end{align}
A straightforward observation is that this generalization bound holds \emph{with probability one} over the distribution of size-$n$ random samples. This is a stronger result than usual high-probability bounds. Of course one may derive a high-probability bound from \cref{eq:square_loss_bd} by an application of Markov's inequality, but that would make the result weaker. The stronger result with probability one, for instance, allows to select the best out a countable collection of $\lambda$ values at no extra cost, while the high-probability bound would need to pay a union bound price for such selection.

Notice that we are not necessarily assuming bounded inputs or labels. 
Our bound depends on the data-generating distribution (possibly of unbounded support) via the spectra of the covariance matrices.
While this is apparent by looking at the last term in~\cref{eq:square_loss_bd}, in fact the $\KL(\text{Posterior} \Vert \text{Prior})$ term also depends on the covariances (see Proposition~\ref{prop:KL} in \cref{app:lease_squares}).
In particular, if the data inputs are independent sub-gaussian random vectors, then with high probability $|\hat{\lambda}_i - \lambda_i| \lesssim \sqrt{d/n}$ and the last term in~\cref{eq:square_loss_bd} then behaves as $d \log\bigl( \lambda / (\lambda + \hat{\lambda}_i - \lambda_i)\bigr) \lesssim d/\sqrt{n-1}$.
This of course can be extended to heavy-tailed distributions or, in general, to any input distributions such that spectrum of the covariance matrix concentrates well~\citep{vershynin2011introduction}.

The explicit dependence on the spectrum of the sample covariance matrix opens interesting venues for distribution-dependent analysis.
The above argument can be extended to heavy-tailed data distributions, where in some cases we can have concentration of the smallest eigenvalue of a sample covariance matrix even for unbounded instances, see~\citet[Section 5.4.2]{vershynin2011introduction}.
Moreover, our technique allows to combine PAC-Bayes analysis with specific applications by considering various data distributions.
For instance, we can obtain bounds for structured data by analyzing eigenvalues of the corresponding (sparse or blocked) covariance matrices~\citep{wainwright2019high}, thus revealing fined-grained dependence on the distribution compared to the usual PAC-Bayes bounds.
Similarly, one can obtain generalization bounds for statistically dependent data by looking at the concentration of the covariance with dependent observations~\citep{delapena2012decoupling}.
%


An important component of the proof of \cref{eq:square_loss_bd} is the following identity for the exponential moment of $f = \gamma(L(w) - \hL_S(w))$ under the prior distribution: for $\lambda > \max_i\{\lambda_i - \lambdah_i\}$,
with probability one over random samples $S$,
\begin{equation}
  \label{eq:ls_exponential_moment}
    \log Q^0_{\gamma,\lambda} [e^{f}] 
    =
    \gamma \min_{w \in \R^d} \pr{L(w) - (\hL_S(w) + \frac{\lambda}{2} \|w\|^2)}
    +
    \frac12 \sum_{i=1}^d \log\pr{\frac{\lambda}{\lambda + \lambdah_i - \lambda_i}}~.
  \end{equation}
This identity computes explicitly the exponential moment of $f$ under the prior distribution. Also this explains why the upper bound in \cref{eq:square_loss_bd} contains the term $\min_{w \in \R^d} L(w)$. The latter should be understood as the label noise. This term will disappear in a noise-free problem, while given a distribution-dependent boundedness of the loss function, the term will concentrate well around zero (see Proposition~\ref{prop:gap_bound} in \cref{app:lease_squares}).
We comment on the free parameter $\gamma$ in \cref{app:lease_squares}.

Finally, note that~\cref{eq:ls_exponential_moment} elucidates an equivalence between the concentration of eigenvalues of the sample covariance matrix and concentration of the empirical loss.
Indeed, for simplicity assuming a noise-free setting (that is $\min_{w \in \R^d} L(w) = 0$), we observe that whenever $(\lambdah_i - \lambda_i) \to 0$ as $n \to \infty$ for i.i.d.\ instances, we have $\hL_S(w) \to L(w)$.
This provides an alternative way to control the concentration, compared to works based on restrictions on the loss as e.g. by~\citet{germain2016pac,holland2019pac}.
We discuss another PAC-Bayes bound for unbounded losses in~\cref{sec:free_range}.

%


\section{Our PAC-Bayes theorem for stochastic kernels} 
\label{s:pb-theory}

The following results involve data- and hypothesis-dependent
functions $f : \cS\times\cH \to \R$.
Notice that the order $\cS\times\cH$ is immaterial---functions $\cH\times\cS \to \R$ are treated the same way. It will be convenient to define $f_s(h) = f(s,h)$.
If $\rho \in \cM_1(\cH)$ is a `data-free' distribution, we will write $\rho[f_s]$ to denote the $\rho$-average of $f_s(\cdot)$ for fixed $s$, that is,  $\rho[f_s] = \int_{\cH} f_s(h) \rho(dh)$. 
When $\rho$ is data-dependent, that is, $\rho \in \SK(\cS,\cH)$ is a stochastic kernel, we will write $\rho_s$ for the distribution over $\cH$ corresponding to a fixed $s$, so $\rho_s(B) = \rho(s,B)$ for $B \in \Sigma_{\cH}$, and $\rho_s[f_s] = \int_{\cH} f_s(h) \rho_s(dh)$. 

The joint distribution over $\cS\times\cH$ defined by $P \in \cM_1(\cS)$ and $Q \in \SK(\cS,\cH)$ is the measure denoted%
\footnote{The notation $P \otimes Q$ (see e.g. \cite{kallenberg2017}), used here for the joint distribution over $\cS\times\cH$ defined by $P \in \cM_1(\cS)$ and $Q \in \SK(\cS,\cH)$, corresponds to what in Bayesian learning is commonly written $Q_{H|S} P_S$.} 
by $P \otimes Q$ that acts on functions $\phi : \cS\times\cH \to \R$ as follows:
\begin{align*}
    (P \otimes Q)[\phi] 
    = \int_{\cS} P(ds) \int_{\cH} Q(s,dh) [\phi(s,h)]
    = \int_{\cS}\int_{\cH} \phi(s,h) Q_s(dh) P(ds)\,.
\end{align*}
Drawing a random pair $(S,H) \sim P \otimes Q$ is equivalent to drawing $S \sim P$ and drawing $H \sim Q_S$. In this case, with $\EE$ denoting the expectation under the joint distribution $P \otimes Q$, the previous display takes the form $\EE[\phi(S,H)] = \EE[\EE[\phi(S,H)|S]]$.
Our basic result is the following theorem.

\begin{theorem}[Basic PAC-Bayes inequality]
\label{basic_pb_ineq}
Fix a probability measure $P \in \cM_1(\cS)$,
a stochastic kernel $Q^0 \in \SK(\cS,\cH)$, 
and a measurable function $f : \cS\times\cH \to \R$,
and let
\begin{align*}
\xi = 
    \int_{\cS} \int_{\cH} e^{f(s,h)} Q^0_s(dh) P(ds)\,.
\end{align*}
\begin{itemize}[leftmargin=25pt]
    \item[(i)] For any $Q \in \SK(\cS,\cH)$,
    for any $\delta \in (0,1)$,
    with probability at least $1-\delta$ over the random draw of a pair $(S,H) \sim P \otimes Q$ we have
    \begin{align*}
    f(S,H)
    \leq \log\pr{ \frac{dQ_S}{dQ^0_S}(H) }
    + \log(\xi/\delta)\,.
    \end{align*}
    \item[(ii)] For any $Q \in \SK(\cS,\cH)$,
    for any $\delta \in (0,1)$,
    with probability at least $1-\delta$ over the random draw of $S \sim P$ we have 
    \begin{align*}
    Q_S[f_S] \leq
    \KL(Q_S \Vert Q^0_S) + \log(\xi/\delta)\,.
    \end{align*}
\end{itemize}
\end{theorem} 

To the best of our knowledge, this theorem is new. 
Notice that $Q^0$ is by default a stochastic kernel from $\cS$ to $\cH$. Hence, given data $S$, the prior $Q^0_S$ is a data-dependent distribution over hypotheses. By contrast, the usual PAC-Bayes approaches assume that $Q^0$ is a `data-free' distribution. 
Also note that the function $f$ is unrestricted,
and the distribution $P \in \cM_1(\cS)$ is unrestricted,
except for integrability conditions to ensure that $\xi$ is finite.
A key step of the proof involves a well-known change of measure that can be traced back to \cite{Csiszar1975divergence} and \cite{DoVa1975}.

\begin{proof}
Recall that when $Y$ is a positive random variable, by Markov inequality, for any $\delta \in(0,1)$, with probability at least $1-\delta$ 
we have:
\begin{align*}
    \log Y \leq \log \EE[Y] 
    + \log (1/\delta)\,.
\tag{$\star$}
\label{eq:markov}
\end{align*}
Let $Q^0 \in \SK(\cS,\cH)$, and let $\EE^0$ denote expectation under the joint distribution $P \otimes Q^0$. 
Thus if $S \sim P$ and $H \sim Q^0_S$ we then have
$\xi = \EE^0[ \EE^0[e^{f(S,H)} | S] ]$.

Let $Q \in \SK(\cS,\cH)$ and denote by $\EE$ the expectation under the joint distribution $P \otimes Q$. 
Then by a change of measure we may re-write $\xi = \EE^0[e^{f(S,H)}]$ as $\xi = \EE[e^{\tilde{f}(S,H)}] = \EE[e^{D}]$ with
\begin{align*}
    D = \tilde{f}(S,H) 
    = f(S,H) - \log\pr{ \frac{dQ_S}{dQ^0_S}(H) }\,.
\end{align*}

(i) Applying inequality \eqref{eq:markov} to $Y = e^{D}$, with probability at least $1-\delta$ over the random draw of the pair $(S,H) \sim P \otimes Q$ we get
$D \leq \log\EE[e^{D}] + \log (1/\delta)$.

(ii) Recall $f_S(H) = f(S,H)$. Notice that
$
    \EE[D | S] 
    = Q_S[f_S] 
    - \KL(Q_S \Vert Q^0_S)\,.
$
By Jensen inequality, $\EE[D | S] \leq \log\EE[e^D | S]$.
While from \eqref{eq:markov} applied to $Y = \EE[e^D | S]$, with probability at least $1-\delta$ over the random draw of $S \sim P$ we have
$
    \log\EE[e^D | S]
    \leq \log\EE[e^D]
    + \log (1/\delta)
$.
\end{proof}

Suppose the function $f$ is of the form $f = F \circ A$
with $A : \cS\times\cH \to \R^k$ and $F : \R^k \to \R$ convex. In this case, by Jensen inequality we have $F(Q_s[A_s]) \leq Q_s[F(A_s)]$ and \cref{basic_pb_ineq}(ii) gives:

\begin{theorem}[PAC-Bayes for stochastic kernels]
\label{pbb-general}
For any $P \in \cM_1(\cS)$,
for any $Q^0 \in \SK(\cS,\cH)$, 
for any positive integer $k$,
for any measurable function $A : \cS\times\cH \to \R^k$ and 
convex function $F : \R^k \to \R$, 
let $f = F \circ A$ and let $\xi = (P \otimes Q^0)[e^{f}]$ as in \cref{basic_pb_ineq}.
Then for any $Q \in \SK(\cS,\cH)$ and any $\delta \in (0,1)$,
with probability at least $1-\delta$ over the random draw of $S \sim P$ we have 
\begin{align}
    F(Q_S[A_S]) \leq
    \KL(Q_S \Vert Q^0_S) + \log(\xi/\delta)\,.
\label{eq:pbb}
\end{align}
\end{theorem}
This theorem is a general template for deriving PAC-Bayes bounds, not just with `data-free' priors, but also more generally with data-dependent priors. 
Previous works (see \cref{s:literature} below) that presented similar generic templates for deriving PAC-Bayes bounds only considered data-free priors.
%
%
We emphasize that a `data-free' distribution is equivalent to a constant stochastic kernel: $Q^0_{s} = Q^0_{s^\prime}$ for all $s,s^{\prime} \in \cS$. Hence $\cM_{1}(\cH) \subset \cK(\cS,\cH)$, which implies that our Theorem~\ref{pbb-general} encompasses the usual PAC-Bayes inequalities with data-free priors in the literature.

Interestingly, our Theorem~\ref{pbb-general} is valid with any normed space instead of $\R^k$. 
This theorem extends the typically used case where $k=2$ and $A=(L(h),\hL(h,s))$,
in which case the function of interest is $f(s,h) = F(L(h),\hL(h,s))$, where $F: \R^2 \to \R$ is a convex function, but there are no restrictions on the loss function $\ell$ that is used in defining $L(h)$ and $\hL(h,s)$. Hence Theorem~\ref{pbb-general} is valid for \emph{any} loss function: convex or non-convex, bounded or unbounded. 
Notice also that our Theorem~\ref{pbb-general} holds for any $P \in \cM_1(\cS)$, i.e. without restrictions on the data-generating process. 
In particular, our Theorem~\ref{pbb-general} holds without the i.i.d. data assumption,
hence this theorem could potentially enable new generalization bounds for statistically dependent data.
In \cref{s:literature} below we comment on some literature related to unbounded losses and non-i.i.d. data.

An important role is played by $\xi$, the exponential moment (moment generating function at 1) of the function $f$ under the joint distribution $P \otimes Q^0$.
As discussed above in \cref{s:contributions}, there are essentially two main steps involved in obtaining a PAC-Bayesian inequality: 
(i) choose $F$ to use in Theorem~\ref{pbb-general}, and 
(ii) upper-bound the exponential moment $\xi$. 
We emphasize that the ``usual assumptions'' on which PAC-Bayes bounds are based, namely, (a) data-free prior, (b) bounded loss, and (c) i.i.d. data, played a role only in the technique used for controlling  $\xi$. This is because with a data-free $Q^0$ we may swap the order of integration:
\begin{align*}
    \xi = \int_{\cS} \int_{\cH} e^{f(s,h)} Q^0(dh) P(ds)
        = \int_{\cH} \int_{\cS} e^{f(s,h)} P(ds) Q^0(dh)
        =: \xi_{\mathrm{\tiny swap}}\,.
\end{align*}
Then bounding $\xi$ proceeds by calculating or bounding $\xi_{\mathrm{\tiny swap}}$ for which there are readily available techniques for bounded loss functions and i.i.d. data (see e.g. \cite{maurer2004note}, \cite{germain-etal2009}, \cite{vanerven2014mini}). 
The bounds with data-dependent priors that we presented in \cref{s:contributions} required different kinds of techniques to control the exponential moment, the details are in the appendices. To the best of our knowledge, ours is the first work to extend the PAC-Bayes analysis to stochastic kernels.
This framework appears to be a promising theoretical tool to obtain new results.
The three types of data-dependant priors discussed in \cref{s:contributions} show the versatility of the approach.
Deriving more cases of PAC-Bayes inequalities without the usual assumptions is left for future research.

\section{Additional discussion and related literature}
\label{s:literature}

The literature on the PAC-Bayes learning approach is vast.
%
%
We briefly mention the usual references
\citet{McAllester1999}, \citet{LangfordSeeger2001}, \citet{seeger2002}, and \citet{catoni2007}; but see also 
\cite{maurer2004note}, and \citet{Keshet-etal2011}. 
Note that \cite{McAllester1999} continued \cite{McAllester1998} whose work was inspired by \cite{Shawe-TaylorWilliamson1997}'s work on a PAC analysis of a Bayesian-style estimator.
We acknowledge the tutorials of \cite{Langford2005} and \cite{McAllester2013}, 
the mini-tutorial of \citet{vanerven2014mini},
and the primer of \citet{guedj2019primer}.
Our \cref{pbb-general} is akin to general forms of the PAC-Bayes theorem 
given before by \cite{audibert2004better},
\cite{germain-etal2009}, and \cite{begin2014pac,begin2016pac}.
Our \cref{basic_pb_ineq}(i) is akin to the ``pointwise'' bound of \cite{blanchard2007occam}, in that the bound holds over the draw of data and hypothesis pairs.

There are many application areas that have used the PAC-Bayes approach, but there are essentially two ways that a PAC-Bayes bound is typically applied: either use the bound to give a risk certificate for a randomized predictor learned by some method, or turn the bound itself into a learning method  by searching a randomized predictor that minimizes the bound.
%
The latter is mentioned already by \citet{McAllester1999}, credit for this approach in various contexts is due also to \citet{germain-etal2009}, \citet{SeldinTishby2010}, \citet{Keshet-etal2011}, \citet{NoyCrammer2014robust}, \citet{Keshet-etal2017},
possibility among others.
Recently, the use of the latter approach has also found success in training neural networks, see \citet{dziugaite2017computing,dziugaite2018data}. 
In fact, the recent resurgence of interest in the PAC-Bayes approach has been to a large extent motivated by the interest in 
generalization guarantees for neural networks. %
\cite{LangfordCaruana2001} used \cite{McAllester1999}'s classical PAC-Bayesian bound to evaluate the error of a (stochastic) neural network classifier.
\cite{dziugaite2017computing} obtained numerically non-vacuous generalization bounds by optimizing the same bound. Subsequent studies (e.g. \citet{rivasplata2019pac,perez-ortiz2020tighter}) continued this approach,
sometimes with links to the generalization of stochastic optimization methods (e.g. \cite{London2017}, \cite{neyshabur2018pac}, \cite{dziugaite2018entropy})
or algorithmic stability.

A line of work related to connecting PAC-Bayes priors to data was explored by~\cite{lever-etal2013,pentina2014pac} and more recently by \cite{rivasplata2018pac}, who assumed that priors are \emph{distribution-dependent}.
In that setting the priors are still `data-free' but in a less agnostic fashion (compared to an arbitrary fixed prior), 
which allows to demonstrate improvements for ``nice'' data-generating distributions.
Data-dependent priors were investigated recently by~\cite{awasthi2020pac},
who relied on tools from the empirical process theory and controlled the capacity of a data-dependent hypothesis class (see also~\citet{foster2019hypothesis}).
The PAC-Bayes literature does contain a line of work that investigates 
relaxing the restriction of bounded loss functions.
A straightforward way to extend PAC-Bayes inequalities to unbounded loss functions is to make assumptions on the tail behaviour of the loss  \citep{alquier2016properties,germain2016pac}
or its moments
\citep{alquier2018simpler,holland2019pac},
leading to interesting bounds in special cases.
Recent work has also looked into the analysis for heavy-tailed losses. For example, \cite{alquier2018simpler} proposed a polynomial moment-dependent bound with $f$-divergence replacing the KL divergence, while \cite{holland2019pac} devised an exponential bound assuming that the second moment of the loss is bounded uniformly across hypotheses.
An alternative approach was explored by~\cite{kuzborskij2019efron}, who proposed a stability-based approach by controlling the Efron-Stein variance proxy of the loss.
Squared loss regression was studied by \cite{shalaeva2020improved} who improved results of \cite{germain2016pac} and also relaxed the data-generation assumption to non-iid data.
It is worth mentioning the important work related to extending the PAC-Bayes framework to statistically dependent data, see e.g. \citet{AlquierWitenberger2012} who applied \cite{Rio2000}'s version of Hoeffding's inequality, derived PAC-Bayes bounds for non-i.i.d. data, and used them in model selection for time series.

As we mentioned in the introduction, besides randomized predictions,
other prediction schemes may be derived from a learned distribution over hypotheses.
Aggregation by exponential weighting was considered by \citet{DalalyanTsybakov2007,DalalyanTsybakov2008},
ensembles of decision trees were considered by \cite{Lorenzen2019},
weighted majority vote by \cite{Masegosa2020,germain2015risk}.
This list is far from being complete.
%
Finally, it is worth mentioning that the PAC-Bayesian analysis extends beyond bounds on the gap between population and empirical losses:
A large body of literature has also looked into upper and lower bounds on the \emph{excess risk}, namely, $Q_S[L] - \inf_{h \in \cH} L(h)$, we refer e.g.\ to \cite{catoni2007,alquier2016properties,grunwald2019tight,kuzborskij2019distribution,mhammedi2019pac}.
The approach of analyzing the gap (for randomized predictors), which we follow in this paper, is generally complementary to such excess risk analyses.

\newpage

\section*{Broader Impact}

We think this work will have a positive impact on the theoretical machine learning community. However, since this work presents a high-level theoretical framework, its direct impact on society will be linked to the particular user-specific applications where this framework may be instantiated.

\begin{ack}
We warmly thank the anonymous reviewers for their valuable feedback, which helped us to  improve the paper greatly. For comments on various early parts of this work we warmly thank Tor Lattimore, Yevgeny Seldin, Tim van Erven, Benjamin Guedj, and Pascal Germain. We warmly acknowledge the Foundations team at Deepmind, and the AI Centre at University College London, for providing friendly and stimulating work environments. Omar Rivasplata and Ilja Kuzborskij warmly thank Vitaly Feldman for interesting discussions and a fun table tennis game while visiting DeepMind.

Omar Rivasplata gratefully acknowledges DeepMind sponsorship for carrying out research studies on the theoretical foundations of machine learning and AI at University College London. This work was done while Omar was a research scientist intern at DeepMind.

Csaba Szepesv{\'a}ri gratefully acknowledges funding from the Canada CIFAR AI Chairs Program, the Alberta Machine Intelligence Institute (Amii), and the Natural Sciences and Engineering Research Council (NSERC) of Canada.

John Shawe-Taylor gratefully acknowledges support and funding from the U.S. Army Research Laboratory and the U. S. Army Research Office, and by the U.K. Ministry of Defence and the U.K. Engineering and Physical Sciences Research Council (EPSRC) under grant number EP/R013616/1.
\end{ack}

\bibliography{biblia-theone}


\renewcommand{\theHsection}{A\arabic{section}}

\newpage

\makeatletter
\providecommand{\maketitleappendix}{}
\renewcommand{\maketitleappendix}{%
  \par
  \begingroup
    \renewcommand{\thefootnote}{\fnsymbol{footnote}}
    \renewcommand{\@makefnmark}{\hbox to \z@{$^{\@thefnmark}$\hss}}
    \long\def\@makefntext##1{%
      \parindent 1em\noindent
      \hbox to 1.8em{\hss $\m@th ^{\@thefnmark}$}##1
    }
    \thispagestyle{empty}
    \@maketitle
    \@thanks
    \@notice
  \endgroup
  \let\maketitle\relax
  \let\thanks\relax
}
\title{PAC-Bayes Analysis Beyond the Usual Bounds: Supplementary Material}
\maketitleappendix
\appendix

\vspace{-2mm}
\section{Measure-Theoretic Notation}
\label{measurability}

Let $(\cX,\Sigma_{\cX})$ be a measurable space, i.e.
$\cX$ is a non-empty set and $\Sigma_{\cX}$ is a sigma-algebra of subsets of $\cX$. A measure 
is a countably additive set function $\nu : \Sigma_{\cX} \to [0,+\infty]$ such that $\nu(\varnothing) = 0$.
We write $\cM(\cX,\Sigma_{\cX})$ for the set of all measures on this space, and $\cM_1(\cX,\Sigma_{\cX})$ for the set of all measures with total mass 1, i.e. probability measures.
Actually, when the sigma-algebra where the measure is defined is clear from the context, the notation may be shortened to $\cM(\cX)$ and $\cM_1(\cX)$, respectively.
For any 
measure $\nu \in \cM(\cX)$ and measurable function $f : \cX \to \mathbb{R}$, we write $\nu[f]$ to denote
the $\nu$-integral of $f$, so
\begin{align*}
\nu[f] = \int_{\cX} f(x) \nu(dx)\,.
\end{align*}
Thus for instance if $X$ is an $\cX$-valued random variable
with probability distribution $P \in \cM_1(\cX)$,
i.e. for sets $A \in \Sigma_{\cX}$
the event that the value of $X$ falls within $A$
has probability $\PP[X \in A] = P(A)$.
Then the expectation of $f(X)$ is $\EE[f(X)] = P[f]$, and its variance is $\Var[f(X)]=P[f^2]-P[f]^2$.

\section{Proof of the bound for data-dependent Gibbs priors}
\label{proof_xi_empirical_gibbs}

For the sake of clarity let us recall once more that $P \otimes Q$ denotes the joint distribution over $\cS\times\cH$ defined by $P \in \cM_1(\cS)$ and $Q \in \SK(\cS,\cH)$. Drawing a random pair $(S,H) \sim P \otimes Q$ is equivalent to drawing $S \sim P$ and drawing $H \sim Q_S$.
With $\EE$ denoting expectation under $P \otimes Q$, for measurable functions $\phi : \cS\times\cH \to \R$ we have $\EE[\phi(S,H)] = \EE[\EE[\phi(S,H)|S]]$. Also recall $\cS = \cZ^n$.

\begin{lemma}
\label{mgf-gibbs-erm}
For any $n$, 
for any loss function with range $[0,b]$, 
for any $Q \in \SK(\cS,\cH)$ 
such that $Q_{s}(dh) \propto e^{-\gamma\hL(h,s)} \mu(dh)$, 
the following upper bound on $\xi(Q) = \EE[e^{\sqrt{n} \left( L(H) - \hL(H,S) \right)}]$ holds:
\begin{align*}
    \log(\xi(Q)) \leq 2b^2 \Bigl(1 + \frac{2\gamma}{\sqrt{n}} \Bigr) + \log\pr{1 + e^{b^2/2}}\,.
\end{align*}
\end{lemma}
For the proof of Lemma \ref{mgf-gibbs-erm}, we will use the shorthand 
$\Delta_s(h) = \sqrt{n} \bigl( L(h) - \hL(h,s) \bigr)$ where $(s,h) \in \cS\times\cH$. We need two technical results, quoted next for convenience.

\begin{lemma}[\protect{\citealt[Lemma~4.18]{boucheron2013concentration}}]
  \label{lem:transportation}
  Let $Z$ be a real-valued integrable random variable such that
  \[
    \log \EE\br{e^{\alpha (Z - \EE[Z])}} \leq \frac{\alpha^2 \sigma^2}{2} 
    \qquad (\forall\alpha > 0)
  \]
  holds for some $\sigma > 0$, and let $Z'$ be another real-valued integrable random variable.
  Then we have
  $
    \EE[Z'] - \EE[Z] \leq \sqrt{2 \sigma^2 \KL\pr{ \Law(Z') \Vert \Law(Z)}}
  $.
\end{lemma}

\begin{lemma}[\protect{\citealt[Lemma~9]{kuzborskij2019distribution}}]
  \label{lem:ln_N_Z_bound}
  Let $f_A, f_B : \cH \to \R$ be measurable functions such that the normalizing factors
  \[
    N_A = \int_{\cH} e^{-\gamma f_A(h)} \diff h
    \hspace{5mm}\text{ and }\hspace{5mm}
    N_B = \int_{\cH} e^{-\gamma f_B(h)} \diff h
  \]
  are finite for all $\gamma > 0$,
  and let $p_A$ and $p_B$ be the corresponding densities:
  \[
    p_A(h) = \frac{1}{N_A} \, e^{-\gamma f_A(h)}\,, \hspace{7mm}
    p_B(h) = \frac{1}{N_B} \, e^{-\gamma f_B(h)}\,, \hspace{7mm}  
    h \in \cH~.
  \]
  Whenever $N_A > 0$ we have that
  \[
    \log\pr{\frac{N_B}{N_A}}
    \leq
    \gamma \int_{\cH} p_B(h) \pr{f_A(h) - f_B(h)} \diff h~.
  \]
\end{lemma}

The last lemma is helpful for bounding the log-ratio of Gibbs integrals.
The notation `$\diff h$' stands for integration with respect to a fixed reference measure (suppressed in the notation) over the space $\cH$.
Now we are ready for the proof.

\bigskip

\begin{proof}[of Lemma \ref{mgf-gibbs-erm}]
Throughout the proof we will use an auxiliary random variable $H'$ drawn randomly from a distribution $Q' \in \cM_1(\cH)$ that does not depend on $S$ in any way.
The first step is to relate the exponential moment of $\Delta_S(H)$ to the expectation of $\Delta_S(H)$ under a suitably defined Gibbs distribution and the exponential moment of $\Delta_S(H')$.
Then the expectation of $\Delta_S(H)$ will be bounded via an \emph{algorithmic stability} analysis of the Gibbs density as in the proof of Theorem~1 by~\cite{kuzborskij2019distribution}, while the exponential moment of $\Delta_S(H')$ is bounded by readily available techniques since the distribution of $H'$ is decoupled from $S$.
  
We will carry out the first step through the continuous version of the log-sum inequality, which says that for positive random variables $A$ and $B$ one has:
\begin{align*}
    \EE[A] \log \frac{\EE[A]}{\EE[B]}
    &\leq
      \EE\br{ A \log\pr{\frac{A}{B}} }~.
\end{align*}
We will use this inequality with the random variables $A = e^{\Delta_S(H)}$ and $B = e^{(\Delta_S(H'))_+}$ where $(x)_+ = x \ind{x \geq 0 }$ is the positive part function. This gives
  \begin{align*}
    \EE\br{e^{\Delta_S(H)} } \pr{\log \EE\br{e^{\Delta_S(H)} } - \log \EE\br{e^{(\Delta_S(H'))_+} }}
    \leq
    \EE\br{e^{\Delta_S(H)} \pr{\Delta_S(H) - (\Delta_S(H'))_+} }
  \end{align*}
so then rearranging
  \begin{align}
    \log \EE\br{e^{\Delta_S(H)} }
    &\leq
      \EE\br{ \frac{e^{\Delta_S(H)}}{\EE\br{e^{\Delta_S(H)} }} \pr{\Delta_S(H) - (\Delta_S(H'))_+} }
      +
      \log \EE\br{e^{(\Delta_S(H'))_+} } \notag \\
    &\leq
      \EE\br{ \frac{e^{\Delta_S(H)}}{\EE\br{e^{\Delta_S(H)} }} \Delta_S(H) }
      +
      \log \EE\br{e^{(\Delta_S(H'))_+} }~.
      \label{eq:log_sum_result}
  \end{align}
  Let's write $q_s$ for the density of $Q_s$ with respect to a reference measure $\diff h$ over $\cH$, and introduce a measure
  \[
    \diff \mu_S(h) = \frac{e^{\Delta_S(h)}}{\EE\br{e^{\Delta_S(H)} }} \, \diff q_S(h), \hspace{7mm} h \in \cH~.
  \]
  Then the inequality \eqref{eq:log_sum_result} can be written as
  \[
    \log \EE\br{e^{\Delta_S(H)} }
    \leq
    \underbrace{
      \EE \int \Delta_S(h) \diff \mu_S(h)
      }_{(I)}
    +
    \underbrace{
      \log \EE\br{e^{(\Delta_S(H'))_+} }
    }_{(II)}~.
  \]

  \paragraph{Bounding $(I)$.}
  We handle the first term through the stability analysis of the density $\mu_S$.
  We will denote by $S\repi = (Z_{1:i-1},Z_1',Z_{i+1:n})$ the sample obtained from $S = (Z_{1:i-1},Z_i,Z_{i+1:n})$ when replacing the $i$th entry with an independent copy $Z_1'$.
  In particular,
  \begin{align}
  \label{eq:E_int_Delta_param}
    \frac{1}{\sqrt{n}} \EE \int \Delta_S(h) \diff \mu_S(h) \nonumber
    &=
      \EE \int \ell(h, Z_1') \diff \mu_S(h) -  \frac{1}{n} \sum_{i=1}^n \EE \int \ell(h, Z_i) \diff \mu_S(h) \nonumber\\
       &=
       \frac{1}{n} \sum_{i=1}^n \EE
       \int \pr{ \ell(h, Z_1') - \ell(h, Z_i) } \diff \mu_S(h)
       \\
    &=
      \frac{1}{n} \sum_{i=1}^n \EE \br{
      \int \ell(h, Z_i) \diff \mu_{S\repi}(h) - \int \ell(h, Z_i) \diff \mu_S(h)}~. \nonumber
  \end{align}  
  The last equality comes from switching $Z_1'$ and $Z_i$ since these variables are distributed identically.
  Now we use Lemma~\ref{lem:transportation} with $\mu_{S\repi}$ and $\mu_S$, and with $\sigma = b$, to get that
  \begin{align*}
    \int \ell(h, Z_i) \diff \mu_{S\repi}(h) - \int \ell(h, Z_i) \diff \mu_S(h)
    \leq
    \sqrt{2 b^2 \KL\pr{\mu_{S\repi} \Vert \mu_S}}~.
  \end{align*}
  Notice that we may use $\sigma = b$ in Lemma~\ref{lem:transportation} since the loss function has range $[0,b]$.
  Focusing on the $\KL$-divergence, and writing `$\diff h$' for a reference measure on $\cH$ with respect to which $q_S$, $\mu_S$, $\mu_{S\repi}$ are absolutely continuous,
  \begin{align*}
    \KL &\pr{\mu_{S\repi} \Vert \mu_S} = 
      \int \log(\diff\mu_{S\repi}(h)/\diff h) \diff \mu_{S\repi}(h)
      -
      \int \log(\diff\mu_S(h)/\diff h) \diff \mu_{S\repi}(h)\\
    &=
      \int \log \pr{ \frac{e^{\Delta_{S\repi}(h)}}{\EE\br{e^{\Delta_S(H)} }} \frac{e^{-\gamma \hL_{S\repi}(h)}}{N_{S\repi}} } \diff \mu_{S\repi}(h)
      -
      \int \log \pr{ \frac{e^{\Delta_S(h)}}{\EE\br{e^{\Delta_S(H)} }} \frac{e^{-\gamma \hL_S(h)}}{N_S} } \diff \mu_{S\repi}(h)\\
    &=
      \int \pr{\Delta_{S\repi}(h) - \Delta_S(h)} \diff \mu_{S\repi}(h)
      +
      \log\pr{\frac{N_S}{N_{S\repi}}}
      +
      \gamma \int \pr{\hL_S(h) - \hL_{S\repi}(h)} \diff \mu_{S\repi}(h)\\
    &\leq
      \sqrt{n} \int \pr{\hL_S(h) - \hL_{S\repi}(h)} \diff \mu_{S\repi}(h) \tag{By definition of $\Delta_S$}\\
    &\hspace*{5mm}+
      \gamma \int \pr{\hL_{S\repi}(h) - \hL_S(h)} \diff \mu_S(h) \tag{By Lemma~\ref{lem:ln_N_Z_bound}}\\
    &\hspace*{10mm}+
      \gamma \int \pr{\hL_S(h) - \hL_{S\repi}(h)} \diff \mu_{S\repi}(h)\\
    &=
      \frac{1}{\sqrt{n}} \int \pr{\ell(h, Z_i) - \ell(h, Z_1')} \diff \mu_{S\repi}(h)\\
    &\hspace*{5mm}+
      \frac{\gamma}{n} \int \pr{\ell(h, Z_1') - \ell(h, Z_i)} \diff \mu_S(h)\\
    &\hspace*{10mm}+
      \frac{\gamma}{n} \int \pr{\ell(h, Z_i) - \ell(h, Z_1')} \diff \mu_{S\repi}(h)~,
  \end{align*}
  where the last step is due to multiple cancellations.
  Therefore, taking expectation,
  \begin{align*}
      \EE[ \KL &\pr{\mu_{S\repi} \Vert \mu_S} ]
      \leq \pr{\frac{1}{\sqrt{n}} + \frac{2 \gamma}{n}}
      \EE\br{
      \int \pr{\ell(h, Z_1') - \ell(h, Z_i)} \diff \mu_S(h)}~.
  \end{align*}
  Putting all together, for each term in~\cref{eq:E_int_Delta_param} (each $i \in [n]$) we get
  \begin{align*}
    &\EE \br{
      \int \pr{\ell(h, Z_1') - \ell(h, Z_i)} \diff \mu_S(h)
      }
    =\EE \br{
      \int \ell(h, Z_i) \diff \mu_{S\repi}(h) - \int \ell(h, Z_i) \diff \mu_S(h)
      }
    \\  
    &\leq
      \EE\br{
      \sqrt{2 b^2 \KL\pr{\mu_{S\repi} \Vert \mu_S}}
      }
      \leq
      \sqrt{2 b^2 \EE[\KL\pr{\mu_{S\repi} \Vert \mu_S}]}
      \tag{By Lemma~\ref{lem:transportation} and Jensen}
    \\
    &=
      \sqrt{
      2 b^2 \pr{\frac{1}{\sqrt{n}} + \frac{2 \gamma}{n}}
      \EE\br{\int \pr{\ell(h, Z_1') - \ell(h, Z_i)} \diff \mu_S(h)}
      }~.
  \end{align*}
  The last calculation implies
  \begin{align*}
    \bigabs{
    \EE \br{\int \pr{\ell(h, Z_i) - \ell(h, Z_i) } \diff \mu_S(h)}
    }
    \leq
    2 b^2 \pr{\frac{1}{\sqrt{n}} + \frac{2 \gamma}{n}}~.
  \end{align*}
  Finally, combining this with~\cref{eq:E_int_Delta_param} gives
  \begin{equation}
    \label{eq:E_int_Delta_stability_bounds}
    \EE \int \Delta_S(h) \diff \mu_S(h)
    \leq
    2 b^2 \pr{1 + \frac{2 \gamma}{\sqrt{n}}}~.
  \end{equation}

  \paragraph{Bounding $(II)$.}
  Now we turn our attention to the exponential moment of $(\Delta_S(H'))_+$  in~\eqref{eq:log_sum_result}:
  \begin{align*}
    \log \EE\br{ e^{(\Delta_S(H'))_+} }
    &=
      \log \EE \EE\br{e^{(\Delta_S(H'))_+} \mid S}\\
    &=
      \log \EE \EE\br{e^{(\Delta_S(H'))_+} \mid H'} \tag{swapping the order of integration}
  \end{align*}
  and observe that the internal expectation is bounded as
  \begin{align*}
    \EE\br{e^{(\Delta_S(H'))_+} \mid H'}
    &\leq
      1 + \EE\br{e^{\Delta_S(H')} \mid H'}\\
    &=
      1 +
      \EE\br{\exp\pr{\frac{1}{\sqrt{n}} \sum_{i=1}^n \pr{\EE[\ell(H', Z_1') \,|\, H'] - \ell(H', Z_i)}} \mid H'}\\
    &=
      1 +
      \prod_{i=1}^n\EE\br{ \exp\pr{\frac{1}{\sqrt{n}} \pr{\EE[\ell(H', Z_1') \,|\, H'] - \ell(H', Z_i)}} \mid H'}
      \\
    &\leq
      1 +
      \prod_{i=1}^n \exp\pr{\pr{2b / \sqrt{n}}^2 / 8}      
      \ = \ 1 + e^{b^2/2}~,
  \end{align*}
  where we obtain the last inequality thanks to Hoeffding's lemma for independent random variables with values in the range $[-b/\sqrt{n}, b/\sqrt{n}]$.
%
%
  Plugging the bounds on terms $(I)$ and $(II)$ into \cref{eq:log_sum_result} finishes the proof of Lemma~\ref{mgf-gibbs-erm}.
\end{proof}
Using Lemma~\ref{mgf-gibbs-erm} to bound $\log(\xi(Q^0))$
we obtain the following corollary by observing that the Gibbs distribution $Q^0$ with density $\propto e^{-\gamma\hL(h,s)}$  satisfies the $\DP(2 \gamma/n)$ property (defined in \cref{sec:DP}).
\begin{corollary}
  \label{cor:gibbs_dp}
  For any $n$,
  for any $P_1 \in \cM_1(\cZ)$,
  for any loss function with range $[0,1]$,
  for any $\gamma > 0$,
  for any $Q^0 \in \SK(\cS,\cH)$ 
  such that $Q^0_{s} \propto e^{-\gamma\hL(h,s)}$,
  for any 
  $Q \in \SK(\cS,\cH)$  and
  $\delta \in (0,1)$,
  with probability at least $1-\delta$ over size-$n$ i.i.d. samples $S \sim P_1^n$
  we have
  \begin{align*}
    \abs{ Q_S[\hL_n] - Q_S[L] }
    \leq
    \sqrt{\frac{\KL(Q_S \Vert Q^0_S)}{2 n}}
      + \frac{\gamma}{n}
      + \sqrt[4]{\frac12 {\textstyle \log(\frac{4}{\delta})}} \, \frac{\sqrt{\gamma}}{n^{\frac34}}
      + \sqrt{\frac{\log\pr{\textstyle \frac{4\sqrt{n}}{\delta}}}{2 n}}~.
  \end{align*}
\end{corollary}
\begin{proof}
Theorem~6 of \cite{mcsherry2007mechanism} gives that the Gibbs distribution $Q^0_s \propto e^{-\gamma \hL(h, s)}$ with potential satisfying
  $\sup_{s,s'}\sup_{h\in \cH}\hL_s(h) - \hL_{s'}(h) \leq 1/n$ for $s,s' \in \cS$ that differ at most in one entry,
  satisfies $\DP(2 \gamma/n)$.
  Combined with Theorem~\ref{thm:pbb-kl-dp}, this gives
  \[
    \kl(Q_S[\hL_S] \Vert Q_S[L])
    \leq
    \frac1n
    \pr{
      \KL(Q_S \Vert Q^0_S)  
      + \frac{2 \gamma^2}{n}
      + \sqrt{2 {\textstyle \log(\frac{4}{\delta})}} \frac{\gamma}{\sqrt{n}}
      + \log\pr{\textstyle \frac{4\sqrt{n}}{\delta}}
    }
  \]  
  and applying Pinsker's inequality $2 (p - q)^2 \leq \kl(p \Vert q)$ we get
  \begin{align*}
    \abs{ Q_S[\hL_S] - Q_S[L] }
    &\leq
      \frac{1}{\sqrt{2 n}}
      \sqrt{
      \KL(Q_S \Vert Q^0_S)  
      + \frac{2 \gamma^2}{n}
      + \sqrt{2 {\textstyle \log(\frac{4}{\delta})}} \frac{\gamma}{\sqrt{n}}
      + \log\pr{\textstyle \frac{4\sqrt{n}}{\delta}}
      }\\
    &\leq
      \sqrt{\frac{\KL(Q_S \Vert Q^0_S)}{2 n}}
      + \frac{\gamma}{n}
      + \sqrt[4]{\frac12 {\textstyle \log(\frac{4}{\delta})}} \, \frac{\sqrt{\gamma}}{n^{\frac34}}
      + \sqrt{\frac{\log\pr{\textstyle \frac{4\sqrt{n}}{\delta}}}{2 n}}~.
  \end{align*}
The last inequality is due to the sub-additivity of $t \mapsto \sqrt{t}$.
\end{proof}

While the argument based on d-stability (i.e. Corollary~\ref{cor:gibbs_dp}) gives a result where the order in $\gamma/n$ matches the one in our bound for the empirical Gibbs prior,
our analysis offers an alternative proof technique that might be of independent interest.

\section{d-stable data-dependent priors and the max-information lemma}
\label{sec:DP}
Let $\pi \in \SK(\cS,\cH)$ be a stochastic kernel. Recall that $\cS=\cZ^n$ is the space of size-$n$ samples. When we say that $\pi$ satisfies the DP property with $\epsilon>0$ (written $\DP(\epsilon)$ for short) we mean that whenever $s$ and $s^{\prime}$ differ only at one element, the corresponding distributions over $\cH$ satisfy:
\begin{align*}
    \frac{d\pi_{s}}{d\pi_{s^{\prime}}} \leq e^{\epsilon}\,.
\end{align*}
This condition on the Radon-Nikodym derivative is equivalent to the condition that, whenever $s$ and $s^{\prime}$ differ at one entry, the ratio $\pi(s,A)/\pi(s^{\prime},A)$
is upper bounded by $e^\epsilon$, for all sets $A \in \Sigma_{\cH}$. 
Thus, the property entails stability of the data-dependent distribution $\pi_{s}$ with respect to small changes in the composition of the $n$-tuple $s$. 
This definition goes back to the literature on privacy-preserving methods for data analysis~\citep{dwork2015preserving}; 
however, we are interested in its formal properties only.
It captures a kind of `distributional stability' which we refer to as `d-stability' for short.

As noted before, the main challenge in obtaining PAC-Bayes bounds is in controlling the exponential moment $\xi(Q^0) = (P_n \otimes Q^0)[e^f]$ for given $P_n \in \cM_1(\cS)$ and $Q^0 \in \SK(\cS,\cH)$.

In the following we rely on a notion of $\beta$-approximate \emph{max-information}~\citep{dwork2015generalization,dwork2015preserving}, denoted $I_{\infty}^{\beta}(X;Y)$ for $\beta>0$ and arbitrary random variables $X \in \cX$ and $Y \in \cY$. Intuitively, this intends to measure the worst-case `distributional distance' of the jointly distributed pair $(X,Y)$ from the pair $(X',Y)$ with $X'$ a copy of $X$ independent from $Y$. 
Formally, $I_{\infty}^{\beta}(X;Y)$ is defined as the least $\eta>0$ such that 
for every $C\in\Sigma_{\cX}\otimes\Sigma_{\cY}$ (the product sigma-algebra) we have
\[
  \P[(X,Y) \in C] 
  \leq e^{\eta} \P[(X',Y) \in C] + \beta~.
\]
Special care is needed in defining $I_{\infty}^{\beta}(X;\psi(X))$, 
i.e. the `distributional distance' of the pair $(X,\psi(X))$ to the independent pair $(X',\psi(X))$.
In our context (see below) we need $I_{\infty}^{\beta}(S;Q^0_S)$.
The next lemma 
generalizes an idea we learned from \cite{dziugaite2018data}:

\begin{lemma}{\em (max-information lemma)} 
\label[lemma]{mit}
Fix $n \in \N$, $P_n \in \cM_1(\cS)$, and a function $f : \cS\times\cH \to \R$.
Let $\zeta(n)$ be a positive sequence (possibly constant).
Suppose that for any data-free distribution $Q^*\in\cM_1(\cH)$, for any kernel $Q \in \SK(\cS,\cH)$ and for any $\delta \in (0,1)$, with probability of at least $1-\delta$ over size-$n$ random samples $S \sim P_n$ the following holds:
\begin{align}
    Q_S[f_S] \leq
    \KL(Q_S \Vert Q^*)+\log(\zeta(n)/\delta) \,.
\label{eq:pbb-fixed-prior}
\end{align}
Then for any kernels $Q^0,Q \in \SK(\cS,\cH)$,
and for any $\delta \in (0,1)$, 
with probability of at least $1-\delta$ over size-$n$ random samples $S \sim P_n$ we have
\begin{align}
    Q_S[f_S] \leq
    \KL(Q_S \Vert Q^0_S)+\log(2\zeta(n)/\delta) +
    I_{\infty}^{\alpha/2}(S;Q^0_S)\,.
\label{eq:pbb-dstable-prior}
\end{align}
\end{lemma}
This lemma gives a general recipe for converting a PAC-Bayes bound with a fixed `data-free' prior (i.e. \cref{eq:pbb-fixed-prior}) into a similar PAC-Bayes bound with a data-dependent prior (\cref{eq:pbb-dstable-prior}).
The choice of $\zeta(n)$ is problem-dependent, but the idea is that
if $\xi(Q^*) = (P_n \otimes Q^*)[e^{f}]$ satisfies $\xi(Q^*) \leq \zeta(n)$
when $Q^*$ is a data-free distribution, then $\zeta(n)$ can be re-used in \cref{eq:pbb-dstable-prior}.
For a given $P_n$ and $f$, the best choice of $\zeta(n)$ would be
$\zeta(n) = \inf_{Q^* \in \cM_1(\cH)} \int \int e^{f(s,h)} Q^*(dh) P_n(ds)$.

The statement of \cref{mit} is written in the generic framework of
\cref{basic_pb_ineq}. We may specialize it to \cref{pbb-general} 
when the function $f$ used in the left hand side of the inequality---and in the exponential moment $\xi(Q^*) = (P_n \otimes Q^*)[e^{f}]$---has the form of a composition $f(s,h) = F(A(s,h))$, with $A : \cS\times\cH \to \R^k$ any measurable function, and  $F : \R^k \to \R$ any convex function.
The literature uses $k=2$ and $A=(L(h),\hL(h,s))$; and various choices of $F$ lead to various PAC-Bayes bounds.
Notice that, by Jensen's inequality, 
$F(Q_s[A_s]) \leq Q_s[F(A_s)] = Q_s[f_s]$ for any $s$.

The following upper bound (see \citet[Theorem 20]{dwork2015generalization}) on the max-information $I_{\infty}^{\beta}(S;Q^0_S)$ is available when the stochastic kernel $Q^0$ satisfies the $\DP(\epsilon)$ property:
\begin{align*}
    I_{\infty}^{\beta}(S;Q^0_S) \leq \frac{n\epsilon^2}{2} + \epsilon\sqrt{\frac{n}{2}\log(\frac{2}{\beta})}\,.
\end{align*}

Therefore, via the max-information lemma, one may derive PAC-Bayes bounds which are valid for d-stable data-dependent priors.
Specific forms of the upper bound can be obtained when a specific $\zeta(n)$ (i.e. a bound on $\xi(Q^*)$) is available.
For instance, for the PAC-Bayes-kl bound, which uses $F(x,y) = n \kl(y\Vert x)$, we may take $\zeta(n) = 2\sqrt{n}$ \citep{maurer2004note}, and obtain the following:

\begin{theorem}
\label{thm:pbb-kl-dp}
For any $n$,
for any $P_1 \in \cM_1(\cZ)$,
for any $Q^0 \in \SK(\cS,\cH)$ satisfying $\DP(\epsilon)$,
for any loss function with range $[0,1]$,
for any $Q \in \SK(\cS,\cH)$,
for any $\delta \in (0,1)$,
with probability at least $1-\delta$ over size-$n$ i.i.d. samples $S \sim P_1^n$
we have
\begin{align}
\kl(Q_S[\hL_S] \Vert Q_S[L]) \leq
\frac{
\KL(Q_S \Vert Q^0_S)
+ \log(\frac{4\sqrt{n}}{\delta})
+ \frac{n\epsilon^2}{2} 
+ \epsilon\sqrt{\frac{n}{2}\log(\frac{4}{\delta})}
}{n} \,.
\label{eq:pbb-dp-max-info}
\end{align}
\end{theorem}
This is Theorem 4.2 of \cite{dziugaite2018data}.
The proof of this theorem takes as starting point the PAC-Bayes-kl bound
\citep{seeger2002,LangfordSeeger2001}, which says that when $Q^*\in\cM_1(\cH)$ is a data-free distribution over hypotheses, 
for any $Q \in \SK(\cS,\cH)$ and any $\delta \in (0,1)$, 
with probability at least $1-\delta$ over size-$n$ i.i.d. samples $S \sim P_1^n$ we have
\begin{align*}
\kl(Q_S[\hL_S] \Vert Q_S[L]) \leq
\frac{
\KL(Q_S \Vert Q^0_S)
+ \log(\xi(Q^*)/\delta)
}{n} \,.
\end{align*}
Notice that this PAC-Bayes-kl inequality follows from
\cref{pbb-general}, which in turn follows from \cref{basic_pb_ineq}(ii), using $f(s,h) = F(L(h),\hL(h,s))$
with $F(x,y) = n \kl(y\Vert x)$
under the restriction of losses within the range $[0,1]$.
Then we may use $\xi(Q^*) \leq 2\sqrt{n}$ since $Q^*$ is a fixed `data-free' distribution (cf. \cite{maurer2004note}). Then use \cref{mit}, and upper-bound the $(\alpha/2)$-approximate max-information as per the inequality of \citet[Theorem 20]{dwork2015generalization} cited before \cref{thm:pbb-kl-dp}.

\subsection{Proof of the max-information lemma}
\label{proof_mit}

Let $f(s,h)$ be a data-dependent and hypothesis-dependent function. 
Recall that $s$ summarizes a size-$n$ sample.
Let $Q^* \in \cM_1(\cH)$ be a fixed `data-free'  distribution over $\cH$, and let $Q \in \SK(\cS,\cH)$ be a stochastic kernel.
Suppose \cref{eq:pbb-fixed-prior} is satisfied (this is the assumption required by \cref{mit}). Given $\delta' \in (0,1)$, define the set
$$
\cE(Q^*) = \big\{ s \in \cS \hspace{2mm}|\hspace{2mm} 
Q_s[f_s] > \KL(Q_s \Vert Q^*)+\log(\zeta(n)/\delta') \big\}~.
$$
Notice that for a random sample $S \sim P_n$ we have
$\P[S \in \cE(Q^*)] = P_n(\cE(Q^*)) \leq \delta'$ by \cref{eq:pbb-fixed-prior}.
Now suppose $Q^0 \in \SK(\cS,\cH)$ is a stochastic kernel,
so each random size-$n$ data set $S$ is mapped to a data-dependent distribution $Q^0_S$ over $\cH$. Correspondingly, define the set
$$
\cE(Q^0) = \big\{ (s,s') \in \cS\times\cS \hspace{2mm}|\hspace{2mm} 
Q_s[f_s] > \KL(Q_s \Vert Q^0_{s'})+\log(\zeta(n)/\delta') \big\}~.
$$
We are interested in the event that a random sample $S \sim P_n$ satisfies $(S,S) \in \cE(Q^0)$.
For fixed $s' \in \cS$, consider the section 
$\cE(Q^0)_{s'} = \{ s \in \cS \hspace{2mm}|\hspace{2mm} (s,s') \in \cE(Q^0) \}$; and notice that $(s,s') \in \cE(Q^0)$ if and only if $s \in \cE(Q^0)_{s'}$.
For any fixed $s'$, the random sample satisfies 
$\P[S \in \cE(Q^0)_{s'}] \leq \delta'$, again by \cref{eq:pbb-fixed-prior}.
Then if $S' \sim P_n$ is an independent copy of $S$, we have
$$
\P[ (S,S') \in \cE(Q^0)] 
= \P[S \in \cE(Q^0)_{S'}]
= \E[\P[S \in \cE(Q^0)_{S'} | S']]
\leq \delta'~.
$$
By the definition of $\beta$-approximate max-information \citep{dwork2015generalization} we have
\begin{align*}
\P[ (S,S) \in \cE(Q^0)]
\leq e^{I_{\infty}^{\beta}(S;Q^0_S)}\P[ (S,S') \in \cE(Q^0)] + \beta 
\leq e^{I_{\infty}^{\beta}(S;Q^0_S)} \delta' + \beta~.
\end{align*}
Therefore, given $\delta \in (0,1)$, setting $\beta = \delta/2$ and $\delta' = e^{-I_{\infty}^{\alpha/2}(S;Q^0_S)}\delta/2$, we get $\P[S \in \cE(Q^0)_S] \leq \delta$.
This finishes the proof of the ``max-information lemma'' (Lemma \ref{mit}).

\paragraph{Remark.}
Let $Q^* \in \cM_1(\cH)$ be a `data-free' distribution,
and suppose the exponential moment 
$\xi(Q^*) = \int \int e^{f(s,h)} Q^*(dh) P_n(ds)$
satisfies $\xi(Q^*) \leq \xi_{\mathrm{\tiny bd}}$.
If a stochastic kernel $Q^0 \in \SK(\cS,\cH)$ satisfies $\DP(\epsilon)$ for some $\epsilon>0$, then in the exponential moment
\begin{align*}
    \xi(Q^0) = \int_{\cS} \int_{\cH} e^{f(h,s)} Q^0_{s}(dh) P_n(ds)
\end{align*}
we may change the measure $Q^0_{s}$ to $Q^0_{s^{\prime}}$ with any fixed $s^{\prime} \in \cS$, and the Radon-Nikodym derivative satisfies $dQ^0_{s}/dQ^0_{s^{\prime}} \leq e^{n\epsilon}$, so we have
\begin{align*}
\xi(Q^0) 
\leq e^{n\epsilon} \int_{\cS} \int_{\cH} e^{f(h,s)} Q^0_{s^{\prime}}(dh) P_n(ds)
\leq e^{n\epsilon} \xi_{\mathrm{\tiny bd}}
\end{align*}
where the integral on the right hand side is upper bounded by $\xi_{\mathrm{\tiny bd}}$ since $Q^0_{s^{\prime}}$ is now a fixed distribution (with respect to the variable $s$ of the outer integral). Thus the max-information lemma gives a refined analysis so that $\log(\xi(Q^0))$ is `replaced' with 
$\log(2\xi_{\mathrm{\tiny bd}})+I_{\infty}^{\delta/2}(S;Q^0_S)$;
whereas the naive argument just described would give 
$ \log(\xi(Q^0)) \leq \log(\xi_{\mathrm{\tiny bd}}) + n\epsilon$.

\section{Proof of the bound for least squares regression}
\label{app:lease_squares}

Let us recall the setting. 
The input space is $\cX = \R^d$ and the label space $\cY = \R$.
A linear predictor is of the form $h_{w}: \R^d \to \R$ with $h_{w}(x) = w\tp x$ for $x \in \R^d$, where of course $w \in \R^d$. Hence we may identify $h_{w}$ with $w$ and correspondingly the hypothesis space $\cH$ may be identified with the weight space $\cW =\R^d$. The size-$n$ random sample is $S = \pr{(X_1, Y_1), \ldots, (X_n, Y_n)} \in (\R^d \times \R)^n$. We are interested in the generalization gap
$\Delta_{w}^{S} = L(w) - \hL_S(w)$, defined for $w \in \R^d$, where
\[
  L(w) = \frac12 \EE[(w\tp X_1 - Y_1)^2] 
  \hspace{7mm}\text{and}\hspace{7mm}
  \hL_S(w) = \frac{1}{2n} \sum_{i=1}^n (w\tp X_i - Y_i)^2
\]
are, respectively, the population and empirical losses under the square loss function.
For $\lambda > 0$, let $\hL_{S,\lambda}(w) = \hL_S(w) + (\lambda/2) \|w\|^2$ be the regularized empirical loss, and $\Delta_{w}^{S,\lambda} = L(w) - \hL_{S,\lambda}(w)$.

The population covariance matrix is $\vSigma = \EE[X_1 X_1\tp] \in \R^{d \times d}$ and its eigenvalues are $\lambda_1 \geq \dots \geq \lambda_d$.
The (regularized) sample covariance matrix is $\hat{\vSigma}_{\lambda} = (X_1 X_1\tp + \dots + X_n X_n\tp)/n + \lambda \vI$ for $\lambda > 0$, with eigenvalues $\hat{\lambda}_1 \geq \dots \geq \hat{\lambda}_d$.

By the well-known change-of-measure (\cite{Csiszar1975divergence}, \cite{DoVa1975}), 
for any (`prior') density $q^0$ the following holds:
\begin{equation}
  \label{eq:DV}
  \int_{\R^d} \Delta_{w}^{S} \, q_S(w) \diff w \leq \KL(q_S \,||\, q^0) + \log \int_{\R^d} e^{\Delta_{w}^{S}} \, q^0(w) \diff w~.
\end{equation}
Note that for simplicity we are saying `density $p(w)$' when in fact what we have in mind is that $p$ is the Radon-Nikodym derivative of a probability $P \in \cM_1(\R^d)$ with respect to Lebesgue measure. i.e. $P(A) = \int_{A} p(w) \diff w$ for Borel sets $A \subset \R^d$.

The main theorem and its proof are as follows. Note that this theorem provides a bound on expected generalization gap, which holds with probability one.
\begin{theorem}
  \label{thm:ls_concentration}
  For any probability kernel $q$ from $\cS$ to $\R^d$, 
  for any $\gamma > 0$ and 
  $\lambda > \max_i\{\lambda_i - \lambdah_i\}$,
  with probability one over random samples $S$,
\begin{align*}
  \int_{\R^d} \Delta_{w}^{S} \, q_S(w) \diff w
  \leq
  \min_{w \in \R^d} \Delta_{w}^{S,\lambda}
  +
  \frac{1}{\gamma} \KL(q_S \,||\, q^0_{\gamma,\lambda})
  +  
  \frac{1}{2 \gamma} \sum_{i=1}^d \log\pr{\frac{\lambda}{\lambda + \lambdah_i - \lambda_i}}~.
\end{align*}
\end{theorem}

\begin{proof}
  We get the statement by combining \cref{eq:DV} with the analytic form of exponential moment of $\gamma \Delta_{w}^{S}$ given by Lemma~\ref{lem:expm_1} below.
\end{proof}

\begin{lemma}[exponential moment]
  \label{lem:expm_1}
  Let $q^0(w) \propto e^{-\frac{\gamma \lambda}{2} \|w\|^2}$ for $\gamma > 0$ and $\lambda > \max_i\{\lambda_i - \lambdah_i\}$.
  Then, with probability one over random samples $S$,
  \begin{align*}
    \log \int_{\R^d} e^{\gamma \Delta_{w}^{S}} \, q^0(w) \diff w
    =
    \gamma \min_{w \in \R^d} \Delta_{w}^{S,\lambda}
    +
    \frac12 \sum_{i=1}^d \log\pr{\frac{\lambda}{\lambda + \lambdah_i - \lambda_i}}~.
  \end{align*}
\end{lemma}
This lemma fills in the main part of the proof of \cref{thm:ls_concentration}.
Notice that this lemma computes explicitly the exponential moment of $\gamma \Delta_{w}^{S}$, without making additional assumptions on the loss function.
The proofs of this lemma and of other results in this section are deferred to \cref{sec:proofs}.

A couple of comments about \cref{thm:ls_concentration}. 
First, note that the inequality holds \emph{almost surely} (a.s.) over samples $S$ which differs from the usual PAC-Bayesian analysis because we did not apply Markov inequality.
However, one can still convert the bound we obtained above to a high-probability bound, by looking at the concentration of eigenvalues of the sample covariance matrix (which will require appropriate assumptions on the marginal distribution).
Second, we have a new term $\min_{w \in \R^d} \Delta_{w}^{S,\lambda}$ whose range is directly connected to that of the loss function.
This term is problem-dependent.
Indeed, the following straightforward proposition lets us understand better its role.
\begin{proposition}[regularized gap]
  \label{prop:gap_bound}
  If $w^\star \in \argmin_{w \in \R^d} L(w)$,
  so that $L(w^\star) = \min_{w \in \R^d} L(w)$,
  then with probability one over random samples $S$ we have that
  \[
    \min_{w \in \R^d} \Delta_{w}^{S,\lambda} \leq L(w^\star)~.
  \]  
  If $\max_i (X_i \tp w^\star - Y_i)^2 \leq B$ a.s., 
  then for any $x > 0$, with probability at least $1-e^{-x}$ we have that
  \[
    \min_{w \in \R^d} \Delta_{w}^{S,\lambda} \leq B \sqrt{\frac{x}{2 n}}~.
  \]
\end{proposition}
The first part of Proposition~\ref{prop:gap_bound} implies that in a noise-free problem the term $\min_{w \in \R^d} \Delta_{w}^{S,\lambda}$ will disappear; while the second part argues that given a distribution-dependent boundedness of the loss function, the term will concentrate well around zero.

Now we turn our attention to the $\KL(\text{Posterior} \Vert \text{Prior})$ term, stated analytically by the following proposition:
\begin{proposition}[KL term]
  \label{prop:KL}
  For $q_S(w) \propto e^{-\frac{\gamma}{2} \Lh_{S,\alpha}(w)}$ and $q^0(w) \propto e^{-\frac{\gamma \lambda}{2} \|w\|^2}$ and any $\alpha, \lambda, \gamma > 0$,
  \[
    \KL(q_S \,||\, q^0)
    =
    \frac12 \pr{
      \log \det\pr{\frac{1}{\lambda} \bhSigma_{\alpha}}
      + \tr\pr{\lambda \bhSigma_{\alpha}^{-1} - \bI}
      + \frac{\lambda \gamma}{n^2} \sum_{i=1}^n Y_i^2 \|X_i\|_{\bhSigma_{\alpha}^{-2}}^2
    }~.
  \]
  Furthermore, if $\max_i \|X_i\|_2 \leq 1$ a.s., then
  \[
    \KL(q_S \,||\, q^0)
    \leq
    \frac12 \pr{
    d \log\pr{\frac{1 + \alpha}{\lambda}}
    + d \pr{\frac{\lambda}{\lambdah_d + \alpha} - 1}
    + \frac{\lambda \gamma}{n^2} \sum_{i=1}^n Y_i^2 \|X_i\|_{\bhSigma_{\alpha}^{-2}}^2
    }~.
  \]
\end{proposition}
Combining the results outlined above yields the following corollary.
\begin{corollary}[data-dependent bound]
  \label{cor:gibbs_ls_bound}
  Let $\sgap = \max_i\{\lambda_i - \lambdah_i\}$, and choose $\lambda = c \sgap$ for some $c > 1$.
  Then, with probability one over random samples $S$,
  \[
    \int_{\R^d} \Delta_{w}^{S} \, q_S(w) \diff w
    \leq
    \min_{w \in \R^d} \Delta_{w}^{S, c\sgap}
    +  
    \frac{d}{2 \gamma} \log\pr{\frac{1 + \alpha}{e (c-1) \sgap}}
    +
    \frac{c \sgap d}{\lambdah_d + \alpha} \pr{ \frac{1}{2 \gamma} + \frac1n \sum_{i=1}^n Y_i^2 }~.
  \]
\end{corollary}


Finally, a quick comment on the free parameter $\gamma>0$ in our bound of \cref{thm:ls_concentration}.
In the standard PAC-Bayes analysis one would see a trade-off in $\gamma$, with a usual near-optimal setting of $\gamma=\sqrt{n}$ \citep{shalaeva2020improved}.
Such trade-off is more subtle in our~\cref{thm:ls_concentration} since one would need to ensure that $\gamma^{-1}\KL(q_S \,||\, q^0_{\gamma,\lambda}) \to 0$ as $\gamma \to \infty$ for the desired choice of $q_S$.

\subsection{Proofs}
\label{sec:proofs}

\begin{proof}[Proof of Lemma~\ref{lem:expm_1}]
For convenience we introduce the abbreviations $\bs = \EE[Y_1 X_1]$ and its empirical counterpart $\bhS = (Y_1 X_1 + \dots + Y_n X_n)/n$. Also let's define $C = \EE[Y_1^2] - (Y_1^2 + \dots + Y_n^2)/n$.
The density is $q^0(w) = Z_0^{-1} e^{-\frac{\gamma \lambda}{2} \|w\|^2}$, with $Z_0$ a normalizing factor.
A straightforward expression of the integral gives
\begin{align}
  \int_{\R^d} e^{\gamma \pr{L(w) - \Lh_S(w)}} q^0(w) \diff w
  &=
    \frac{1}{Z_0} \int_{\R^d} e^{\gamma \pr{L(w) - \Lh_{S,\lambda}(w)}} \diff w \nonumber\\
  &=
    \frac{1}{Z_0} \int_{\R^d} e^{\gamma \pr{C - \frac{1}{2} w\tp (\bhSigma_{\lambda} - \bSigma) w - (\bs - \bhS)\tp w}} \diff w \label{eq:sqloss_identity}\\ 
  &=
    \frac{(2\pi)^{\frac{d}{2}}}{Z_0} \frac{e^{\gamma \pr{C + \frac12 (\bs - \bhS)\tp (\bhSigma_{\lambda} - \bSigma)^{-1} (\bs - \bhS)}}}{\sqrt{\gamma^d \det\pr{\bhSigma_{\lambda} - \bSigma}}} \label{eq:gaussian_int_1}\\
  &=
    \frac{(2\pi)^{\frac{d}{2}}}{Z_0} \frac{e^{\gamma \min_{w \in \R^d}\cbr{L(w) - \Lh_{S,\lambda}(w)}}}{\sqrt{\gamma^d \det\pr{\bhSigma_{\lambda} - \bSigma}}} \label{eq:ls_analytic}\\
  &=
    \sqrt{\frac{\lambda^d}{\det\pr{\bhSigma_{\lambda} - \bSigma}}}
    \,
    e^{\gamma \min_{w \in \R^d}\cbr{L(w) - \Lh_{S,\lambda}(w)}}
    \label{eq:gaussian_int_2}
\end{align}
where \cref{eq:sqloss_identity} is just rewriting things, while in \cref{eq:gaussian_int_1} we assume that $\lambda > \max_i \{\lambda_i - \lambdah_i\}$.
\cref{eq:gaussian_int_1,eq:gaussian_int_2} come from Gaussian integration, and \cref{eq:ls_analytic} is a consequence of:
\begin{proposition}
\label{prop:gap_identity}
  Assuming that $\lambda > \max_i\{\lambda_i - \lambdah_i\}$,
  \[
    \min_{w \in \R^d}\cbr{L(w) - \Lh_{S,\lambda}(w)} = C + \frac{1}{2} (\bs - \bhS)\tp (\bhSigma_{\lambda} - \bSigma)^{-1} (\bs - \bhS)~.
  \]
\end{proposition}
Finally, taking logarithm of the integral completes the proof of Lemma~\ref{lem:expm_1}.
\end{proof}

\begin{proof}[Proof of Proposition~\ref{prop:gap_identity}]
  Observe that
  \begin{align*}
    \nabla_{w} \pr{c - \frac{1}{2} w\tp (\bhSigma_{\lambda} - \bSigma) w - (\bhS - \bs)\tp w}
    =
    - (\bhSigma_{\lambda} - \bSigma) w + (\bs - \bhS)~.
  \end{align*}
  For $\lambda > \max_i\{\lambda_i - \lambdah_i\}$ the matrix $(\bhSigma_{\lambda} - \bSigma)$ is positive definite, and
  plugging the solution of $\nabla_{w}=0$, namely $\hat{w} = (\bhSigma_{\lambda} - \bSigma)^{-1} (\bs - \bhS)$, back into the objective we get
  \begin{align*}
    C - \frac{1}{2} \hat{w}\tp (\bhSigma_{\lambda} - \bSigma) \hat{w} + (\bs - \bhS)\tp \hat{w}
    =
    C + \frac{1}{2} (\bs - \bhS)\tp (\bhSigma_{\lambda} - \bSigma)^{-1} (\bs - \bhS)
  \end{align*}
  which completes the proof of Proposition~\ref{prop:gap_identity}.
\end{proof}

\begin{proof}[Proof of Proposition~\ref{prop:gap_bound}]
  Clearly $\min_{w \in \R^d} \Delta_{w}^{S,\lambda} \leq \Delta_{w^\star}^{S,\lambda} \leq L(w^\star)$,
  which proves the first part of the proposition.
  For the second part, under the assumption that $\max_i (X_i \tp w^\star - Y_i)^2 \leq B$ a.s., Hoeffding's inequality gives:
  \begin{align*}
    \Delta_{w}^{S,\lambda}
    \leq
      \frac12 \EE[(X_1\tp w^\star - Y_1)^2] - \frac{1}{2 n} \sum_{i=1}^n (X_i\tp w^\star - Y_i)^2
    \leq
    B \sqrt{\frac{x}{2 n}}~.
  \end{align*}
  This completes the proof of Proposition~\ref{prop:gap_bound}
\end{proof}

\begin{proof}[Proof of Proposition~\ref{prop:KL}]
  Observe that
  \begin{align*}
    q_S(w)
    &= \frac{e^{-\frac{\gamma}{2} w\tp \bhSigma_{\alpha} w + \gamma w\tp \bhS - \frac{\gamma}{2} \bar{Y}^2}}
    {\int_{\R^d} e^{-\frac{\gamma}{2} u\tp \bhSigma_{\alpha} u + \gamma u\tp \bhS - \frac{\gamma}{2} \bar{Y}^2} \diff u}\\
    &= \frac{e^{-\frac{\gamma}{2} w\tp \bhSigma_{\alpha} w + \gamma w\tp \bhS - \frac{\gamma}{2} \bhS\tp \bhSigma_{\alpha}^{-1} \bhS}}
        {\int_{\R^d} e^{-\frac{\gamma}{2} u\tp \bhSigma_{\alpha} u + \gamma u\tp \bhS - \frac{\gamma}{2} \bhS\tp \bhSigma_{\alpha}^{-1} \bhS} \diff u}\\
    &=: \text{Gauss}(\bhSigma_{\alpha}^{-1} \bhS, \bhSigma_{\alpha}^{-1})  \propto e^{-\frac{\gamma}{2} \pr{w - \bhSigma_{\alpha}^{-1} \bhS}\tp \bhSigma_{\alpha} \pr{w - \bhSigma_{\alpha}^{-1} \bhS}}~,
  \end{align*}
  where
  $\bhS = (Y_1 X_1 + \dots + Y_n X_n)/n$ and $\bar{Y}^2 = (Y_1^2 + \dots + Y_n^2)/n$.
  Recall that analytic form of KL-divergence between two Gaussians is:
  \begin{align*}
    &\KL\pr{ \text{Gauss}(x_1, \bA_1) \,\big\|\, \text{Gauss}(x_0, \bA_0) }\\[1mm]
    &\hspace{7mm}=\frac12 \pr{\log\pr{\frac{\det \bA_0}{\det \bA_1}} + \tr\pr{\bA_0^{-1} \bA_1} - d + (x_1 - x_0)\tp \bA_0^{-1} (x_1 - x_0)}
  \end{align*}
  
  This gives
  \begin{align*}
    \KL\pr{q_S \,||\, q^0}
    &= \frac12 \pr{
      \log \det\pr{\frac{1}{\lambda} \bhSigma_{\alpha}}
      + \tr\pr{\lambda \bhSigma_{\alpha}^{-1} - \bI}
      + \lambda \gamma \, \bhS\tp \bhSigma_{\alpha}^{-2} \bhS
      }
  \end{align*}
  This shows the first statement.

  The `furthermore' statement is shown
  using a simple fact that for $d \times d$ positive definite matrix $\bA$, we have $\det(\bA) \leq (\tr(\bA)/d)^d$,
  \[
    \log \det\pr{\frac{1}{\lambda} \bhSigma_{\alpha}}
    \leq
    d \log \tr\pr{\frac{1}{d \lambda} \bhSigma_{\alpha}}
    \leq
    d \log\pr{\frac{1 + \alpha}{\lambda}}
  \]
  where we have assumed that $\max_i\|X_i\|_2 \leq 1$ a.s. and
  the fact
  \[
    \tr\pr{\lambda \bhSigma_{\alpha}^{-1} - \bI} \leq d \pr{\frac{\lambda}{\lambdah_d + \alpha} - 1}~.
  \]
  This completes the proof of Proposition~\ref{prop:KL}.
\end{proof}

\begin{proof}[Proof of Corollary~\ref{cor:gibbs_ls_bound}]
  Theorem~\ref{thm:ls_concentration} combined with Proposition~\ref{prop:KL} gives us
  \begin{align*}
    \int_{\R^d} \Delta_{w}^{S} \, q_S(w) \diff w
    &\leq
      \min_{w \in \R^d} \Delta_{w}^{S,\lambda}
    +
      \frac{d}{2 \gamma} \log\pr{\frac{1 + \alpha}{\lambda}}
    +
      \frac{d}{2 \gamma} \pr{\frac{\lambda}{\lambdah_d + \alpha} - 1}\\
    &\hspace*{9mm}+
    \frac{\lambda}{2 n^2} \sum_{i=1}^n Y_i^2 \|X_i\|_{\bhSigma_{\alpha}^{-2}}^2
    +  
      \frac{1}{2 \gamma} \sum_{i=1}^d \log\pr{\frac{\lambda}{\lambda + \lambdah_i - \lambda_i}}\\
    &\hspace*{-15mm}\leq
      \min_{w \in \R^d} \Delta_{w}^{S, c\sgap}
    +
      \frac{d}{2 \gamma} \log\pr{\frac{1 + \alpha}{c \sgap}}
    +
      \frac{d}{2 \gamma} \pr{\frac{c \sgap}{\lambdah_d + \alpha} - 1}\\
    &+
      \frac{c d \sgap}{\lambdah_d + \alpha} \pr{\frac1n \sum_{i=1}^n Y_i^2}
      +  
      \frac{d}{2 \gamma} \log\pr{\frac{c}{c-1}}\\
    &\hspace*{-15mm}\leq
      \min_{w \in \R^d} \Delta_{w}^{S, c\sgap}
      +  
      \frac{d}{2 \gamma} \log\pr{\frac{1 + \alpha}{e (c-1) \sgap}}
      +
      \frac{c \sgap d}{\lambdah_d + \alpha} \pr{ \frac{1}{2 \gamma} + \frac1n \sum_{i=1}^n Y_i^2 }~,
  \end{align*}
  where we used the fact that
  \begin{align*}
    \sum_{i=1}^d \log\pr{\frac{\lambda}{\lambda + \lambdah_i - \lambda_i}}
    =
    \sum_{i=1}^d \log\pr{\frac{c \max_i\{\lambda_i - \lambdah_i\}}{c \max_i\{\lambda_i - \lambdah_i\} - (\lambda_i - \lambdah_i)}}
    \leq
      d \log\pr{\frac{c}{c-1}}
  \end{align*}
  and by a simple SVD argument
  \[
    \frac{1}{n^2} \sum_{i=1}^n Y_i^2 \|X_i\|_{\bhSigma_{\alpha}^{-2}}^2 \leq \frac{d}{n (\lambdah_d + \alpha)} \sum_{i=1}^n Y_i^2~.
  \]
  This completes the proof of Corollary~\ref{cor:gibbs_ls_bound}.
\end{proof}

\newpage
\section{A simple PAC-Bayes bound with a `free range' loss function}
\label{sec:free_range}

Consider the case that the loss function $\ell : \cH \times \cZ \to [0,\infty)$ has unbounded range. For any $\lambda>0$ and $h \in \cH$ fixed, we may upper-bound the exponential moment $\EE[\exp\{-\lambda n \hL(h,S)\}]$ using standard techniques under the i.i.d. data-generation model: $S = (Z_1,\ldots,Z_n) \sim P_1^n$. 
Then with $Z \sim P_1$ and a few calculations (shown below in \cref{sec:calc}) 
we obtain:
\begin{align*}
    \EE[e^{ \lambda n ( L(h)-  \hL(h,S) )}]
    \leq e^{ \frac{\lambda^2 n}{2} \EE[\ell(h,Z)^2]}~.
\end{align*}
Assuming $M := \sup_h \EE[\ell(h,Z)^2] < \infty$ (see \cite{holland2019pac} whose main result required this), 
using the function
$f(h,s) = \lambda n \bigl( L(h)-\hL(h,s) \bigr) - \frac{\lambda^2 n}{2}M$,
with a fixed `data-free' prior $Q^0$ 
the exponential moment $\xi = \EE^0[e^{f(S,H)}]$ (i.e. $\xi = P_1^n \otimes Q^0[e^{f}]$) satisfies $\xi \leq 1$. 
This way we obtain the following PAC-Bayes type of bound 
under unbounded (`free range') losses:

\begin{theorem}
\label{thm:pb-freerange}
For any $n$,
for any $P_1 \in \cM_1(\cZ)$,
for any data-free $Q^0 \in \cM_1(\cH)$,
for any loss function 
$\ell : \cH \times \cZ \to [0,\infty)$,
for any $Q \in \SK(\cS,\cH)$,
for any $\lambda\in(0,\infty)$,
for any $\delta \in (0,1)$,
with probability at least $1-\delta$ over size-$n$ i.i.d. samples $S \sim P_1^n$
we have
\begin{align}
    Q_S[L]
\le 
Q_S[\hL_S] + \frac{\KL(Q_S \Vert Q^0)+\log(1/\delta)}{n\lambda} + \frac{\lambda}{2}\sup_h \EE[\ell(h,Z)^2]~.
\label{eq:pb-freerange}
\end{align}
\end{theorem}
%
%
Essentially, this bound is of the form $Q_S[L] - Q_S[\hL_S] \leq B/(n\lambda) + \lambda M/2$. 
With the optimal choice of $\lambda$ we get $Q_S[L] - Q_S[\hL_S] \leq 2\sqrt{BM/(2n)}$, which gives a slow convergence rate of $O(1/\sqrt{n})$.
The assumption of finite $M$ is satisfied e.g. when the loss is sub-gaussian or sub-exponential. It would be interesting to characterize all cases when $M < \infty$ holds.
However, this simple bound illustrates that PAC-Bayes bounds are possible with unbounded loss functions.

\subsection{The calculations to bound the exponential moment}
\label{sec:calc}

We start by calculating 
$\EE[\exp\{-\lambda n \hL(h,S)\}]$
with fixed $\lambda > 0$ and $h \in \cH$. 
This means that the expectation is with respect to $S = (Z_1,\ldots,Z_n) \sim P_1^n$.
By independence, and using the inequality $e^{x} \leq 1 + x + x^2/2$ valid for $x \leq 0$, we have
\begin{align*}
    \EE[\exp\{-\lambda n \hL(h,S)\}]
    &= \prod_{i \in [n]} \EE[\exp\{-\lambda \ell(h,Z_i)\}] \\
    &\leq \prod_{i \in [n]} \EE[ 1 - \lambda \ell(h,Z_i) + \frac{\lambda^2}{2}\ell(h,Z_i)^2] \\
    &= \prod_{i \in [n]} \Bigl( 1 - \lambda \EE[\ell(h,Z_i)] + \frac{\lambda^2}{2} \EE[\ell(h,Z_i)^2] \Bigr)
\intertext{and then using $1 + x \leq e^x$, which is valid for all $x$, the above is}
    &\leq \prod_{i \in [n]} \exp\bigl\{ - \lambda \EE[\ell(h,Z_i)] + \frac{\lambda^2}{2} \EE[\ell(h,Z_i)^2] \bigr\} \\
    &= \exp\bigl\{ - \lambda n L(h) + \frac{\lambda^2 n}{2}\EE[\ell(h,Z)^2] \bigr\}~.
\end{align*}
In the last line we have used the identical distribution of the $Z_i$'s, namely $Z_i \sim P_1$, and a generic identical copy $Z \sim P_1$.
Then, rearranging, we get as claimed that
\begin{align*}
    \EE[e^{ \lambda n ( L(h)-  \hL(h,S) )}]
    \leq e^{ \frac{\lambda^2 n}{2} \EE[\ell(h,Z)^2]}~.
\end{align*}
These kinds of calculations are well known, however, we would like to acknowledge the section `alternative proofs' of  \cite{thiemann2016master}.
Then under the assumption $M = \sup_h \EE[\ell(h,Z)^2] < \infty$, 
using the function $f(h,s) = \lambda n \bigl( L(h)-\hL(h,s) \bigr) - \frac{\lambda^2 n}{2}M$ and a fixed `data-free' prior $Q^0$, 
the exponential moment $\xi$ of $f$ under the joint distribution $P_1^n \otimes Q^0$ 
(i.e. $\xi = P_1^n[Q^0[e^{f}]]$)
satisfies $\xi = \xi_{\mathrm{\tiny swap}} = Q^0[P_1^n[e^{f}]]$ (see discussion after \cref{pbb-general} in \cref{s:pb-theory}), while the above calculations show that the latter satisfies $\xi_{\mathrm{\tiny swap}} \leq 1$.

\end{document}